\documentclass[12pt]{article}

 \usepackage{fullpage}
\usepackage[margin=0.8in]{geometry}  % set the margins to 1in on all sides
\usepackage{amsmath,bm,mathtools,amsfonts,mathrsfs,amssymb,amsthm,mathabx,setspace}              % great math stuff
\usepackage{subfigure,graphicx,color,caption,epstopdf}
\usepackage{float,verbatim,tabularx,enumerate,multirow,enumitem} 
\usepackage{times,xcolor}
\usepackage[OT1]{fontenc}
\usepackage{bbm}% for characteristic function 1 
\usepackage[bookmarks=true,         bookmarksnumbered=true,
colorlinks=true,
pdfstartview=FitV,
linkcolor=blue, citecolor=blue, urlcolor=blue,
]{hyperref}
\usepackage{booktabs}  % for \toprule
\usepackage{algorithm}
\usepackage[noend]{algpseudocode}
\usepackage[normalem]{ulem}

\newtheorem{theorem}{Theorem}[section]
\newtheorem{assumption}[theorem]{Assumption}
\newtheorem{remark}[theorem]{Remark}
\newtheorem*{remark*}{Remark}
\newtheorem{definition}[theorem]{Definition}
\newtheorem{proposition}[theorem]{Proposition}
\newtheorem{lemma}[theorem]{Lemma}

\numberwithin{equation}{section}

\newcommand{\rmref}[1]{{\rm\ref{#1}}}

\newcommand{\wbar}{\widebar}
\newcommand{\rhoT}{\wbar \rho_T}
\newcommand{\GbarT}{\overline G_T}
\newcommand{\barX}{\wbar X}
\newcommand{\HH}{\mathbb{H}}
\newcommand{\rkhsH}{H_{\GbarT}}

\newcommand{\RadiusOmega}{R_\Omega}

\newcommand{\innerp}[2]{ #1\cdot #2 }

\newcommand{\abs}[1]{\big| #1 \big|}

\newcommand{\norm}[1]{\left\| #1 \right\|}

\newcommand{\real}{\mathbb{R}}

\newcommand{\mE}{\mathcal{E}}

\newcommand{\mH}{\mathcal{H}}
\newcommand{\mN}{\mathcal{N}}

\newcommand{\R}{\real}

\newcommand{\intkernel}{\phi}

\newcommand{\hypspace}{\mathcal{H}}
\newcommand{\E}{\mathbb{E}}

\newcommand{\grad}[1]{\nabla #1}
\newcommand{\argmin}[1]{\underset{#1}{\operatorname{arg}\operatorname{min}}\;}

\newcommand{\argmax}[1]{\underset{#1}{\operatorname{arg}\operatorname{max}}\;}
\newcommand{\supp}[1]{\text{supp}(#1)}
\newcommand{\brak}[1]{\left(#1\right)}      % round brackets
\newcommand{\crl}[1]{\left\{#1\right\}}     % curly brackets
\newcommand{\rkhs}[1]{\left<\hspace{-1mm} \left< #1 \right>\hspace{-1mm} \right>}    % double angular brackets
\newcommand{\normmm}[1]{{\left\vert\kern-0.25ex\left\vert\kern-0.25ex\left\vert #1 
    \right\vert\kern-0.25ex\right\vert\kern-0.25ex\right\vert}}
\newcommand{\longinprod}[1]{\left< \kern-0.25ex\left< \kern-0.25ex\left<#1 \right> \kern-0.25ex\right> \kern-0.25ex\right>}    % triple angular brackets(long inner product)

\newcommand{\WOT} {{1,\infty} }  %  {{W^{1, \infty}(\Omega\times[0,T])}}
\newcommand{\WO}{{1,\infty} } % {{W^{1, \infty}(\Omega)}}
\newcommand{\LOT}{{\infty} }% {{L^{\infty}(\Omega\times[0,T])}}
\newcommand{\LO}{{\infty} } % {{L^{\infty}(\Omega)}}
\newcommand{\WBOT}{{2,\infty} } % {{W^{2, \infty}(\Omega\times [0,T])}}
\DeclareMathAlphabet{\mathpzc}{OT1}{pzc}{m}{it}

\title{ Learning interaction kernels in mean-field equations\\ of 1st-order systems of interacting particles }
\author{Quanjun Lang, Fei Lu}
\date{}
\begin{document} 
\maketitle

\abstract{% We consider the inverse problem of estimating the radial interaction kernels of mean-field equations of interacting particles. The data consist of discrete space-time observations of the solution. We introduce a nonparametric algorithm to learn the kernel on data-adaptive hypothesis spaces by least squares with regularization.  We show that the estimator converges in a reproducing kernel Hilbert space, and in an L2 space under an identifiability condition, at a rate optimal in a sense that it equals the order of the numerical integrator involved. We demonstrate our algorithm on three typical examples: the opinion dynamics with a piecewise linear kernel, the granular media model with a quadratic kernel, and the aggregation-diffusion with a repulsive-attractive kernel.}

We introduce a nonparametric algorithm to learn interaction kernels of mean-field equations for 1st-order systems of interacting particles. The data consist of discrete space-time observations of the solution. By least squares with regularization, the algorithm learns the kernel on data-adaptive hypothesis spaces efficiently. A key ingredient is a probabilistic error functional derived from the likelihood of the mean-field equation's diffusion process. 
The estimator converges, in a reproducing kernel Hilbert space and in an L2 space under an identifiability condition, at a rate optimal in the sense that it equals the numerical integrator's order. We demonstrate our algorithm on three typical examples: the opinion dynamics with a piecewise linear kernel, the granular media model with a quadratic kernel, and the aggregation-diffusion with a repulsive-attractive kernel. 
%The error functional converges at a rate depending on the smoothness of the kernel and the numerical integrators approximating the integrals in it. Thus, the rate can be used for model selection. 
 }

\tableofcontents

\begin{comment}
\begin{center}
\textbf{\Large  Interacting particle system and \\
nonlinear Fokker-Planck equation} \\[0pt]
% \vspace{4mm} Fei Lu\\[0pt] feilu@math.jhu.edu \\[0pt]
\end{center}
\end{comment}

\section{Introduction}
We study the inverse problem of estimating the radial interaction kernel $\phi$ of the mean-field equation
  \begin{equation}\label{eq:MFE}
   \begin{aligned}
  \partial_t u &= \nu  \Delta u + \grad\cdot [u  (K_\phi * u)], x\in \R^d, t>0,\\
 u(x,t) & \geq 0, \quad \int_{\R^d} u(x,t)dx = 1,\quad \forall x,t,
\end{aligned}
\end{equation}
from observations of a solution at a sparse mesh in space-time. For simplicity, we assume that the domain $\Omega := \bigcup_{t\in [0,T]} \supp{u(\cdot,t)}\subset \R^d$ is bounded with smooth boundary. Then  $   \partial_t u|_{\partial \Omega} =0, \grad_x u|_{\partial \Omega}=0 $. 
Here $\nu>0$ is a given viscosity constant and $K_\phi:\R^d\to\R^d$ is the gradient of the \emph{radial interaction potential} $\Phi$, whose derivative $\phi$ is called the \emph{interaction kernel}, 
\[ 
K_\phi(x) = \grad ( \Phi(|x|)) =\phi(\abs{x}) \frac{x}{\abs{x}}, \text{ with } \phi(r) =\Phi'(r).
\]
We denote  
$ K_\phi * u(x,t) = \int_{\Omega} K_\phi(x-y) u(y,t)dy$.
Since only the derivative of the potential $\Phi$ affects the equation, we assume, without lost of generality, that the potential satisfies $\Phi(0)=0$.

Equation \eqref{eq:MFE} is the mean-field limits of the 1st-order stochastic interacting particle system 
\begin{equation}\label{eq:sod}
  \frac{d}{dt}{X_t^{i}} =\frac{1}{N} \sum_{i'= 1}^N \intkernel(| X_t^{j} - X_t^{i} |) \frac{X_t^{j} - X_t^{i} }{|X_t^{j} - X_t^{i}|} + \sqrt{2\nu }dB_t^i, \quad \text{for $i = 1, \ldots, N$}
\end{equation}
when $N\to \infty$, 
where $X_t^i$ represents the i-th particle's position (or agent's opinion), and $B_t^i$ is a standard Brownian motion. Such systems arise in many disciplines: particles or molecules in microscopic models in statistical physics and quantum mechanics \cite{golse2016_DynamicsLarge} and in granular media \cite{malrieu2003_ConvergenceEquilibriumc}, cells \cite{keller1970_InitiationSlime,fournier2017_StochasticParticle,carrillo2019_AggregationdiffusionEquations} %the Keller-Segel model 
and neural networks  \cite{baladron2012_MeanfieldDescription} in biology, opinions of agents in social science \cite{motsch2014_HeterophiliousDynamicsa}, and in Monte Carlo sampling \cite{del2013mean}, to name just a few, and we refer to \cite{motsch2014_HeterophiliousDynamicsa,jabin2017_MeanFielda} for the considerable literature.
% mean-field games \cite{ding2020_MeanField}

Motivated by these applications, there has been increasing interest in the inverse problem of estimating the interaction kernel (or the interaction potential) of the mean-field equation. However, except for ideal situations in physics, little information on the interaction kernel is available, which may vary largely from smooth functions in granular media \cite{cattiaux2007_ProbabilisticApproacha} to piece-wise constant function in opinion dynamics \cite{motsch2014_HeterophiliousDynamicsa} or singular kernel in the Keller-Segel model \cite{carrillo2019_AggregationdiffusionEquations}. Thus, it is crucial to develop new methods beyond parametric estimation (see e.g., \cite{fedele2013_InverseProblem}). Towards this direction, recent efforts \cite{bongini2017_InferringInteraction,LZTM19,LMT19,LMT20,zhong2020data} estimate the kernel  by nonparametric regression for systems with finitely many particles from multiple trajectories.  For large systems, data of trajectories of all particles are often unavailable, instead, it is practical to consider data consisting of a macroscopic concentration density of the particles, i.e., the solution of the mean-field equation.  
 
 We introduce a nonparametric scalable learning algorithm (see Algorithm \ref{alg:main}) to estimate the interaction kernel $\phi$ from data with a performance guarantee. The algorithm learns $\phi$ on a data-adaptive hypothesis space by least squares with regularization. A key ingredient is a probabilistic error functional derived from the likelihood of the diffusion process whose Fokker-Planck equation is the mean-field equation (see Theorem \ref{thm:costFn}). The error functional is quadratic, thus we can compute its minimizer by least squares. Furthermore, it does not require spatial derivatives, thus it is suitable for discrete data (see $\mE_{M,L}$ in \eqref{eq:costFn_data}). 

Our estimator converges as the space-time mesh size decreases, in a proper function space, at a rate optimal in the sense that it is almost the same as the order of the numerical integrator evaluating our error functional from data.  More precisely, with space dimension $d=1$, we consider data consisting of a solution observed on space-time mesh: $\{ u(x_m,t_l)\}_{m,l=1}^{M,L}$, where $x_m-x_{m-1}= \Delta x$ and $t_l-t_{l-1} =  \Delta t$. Denote $\hypspace$ a hypothesis space with dimension $n$ and denote $\widehat \phi_n$ the projection of $\phi$ on it. Our estimator $\widehat \phi_{n,M,L} $ in \eqref{eq:mle_nML}, based on Riemann sum approximation to integrals in the Error functional, converges as $(\Delta x,\Delta t)$ decreases (see Theorem \ref{thm:error_discreteTime}),
\[
 \| \widehat \phi_{n,M,L} - \widehat\phi_n\|_{\HH} \leq C(\Delta x  + \Delta t ),  
\]
where the function space $\HH$ is either a reproducing kernel Hilbert space or a weighted $L^2$ space, assuming suitable identifiability conditions and regularity on the solution. The order is the same as the order of the Riemann sum approximation, and it can be improved by using a high-order integrator. We further consider the optimal rate of convergence when $\Delta t=0$ and $\Delta x\to 0$, assuming that we can enlarge the hypothesis space $\hypspace$ to control the approximation error by $\norm{\widehat \phi_n- \phi}_\HH\lessapprox n^{-s}$ with $s\geq 1$. With an optimal $n$ for the trade-off between inference error and the approximation error,
\begin{equation}\label{err_decomposition}
\| \widehat \phi_{n,M,\infty} -\phi\|_{\HH} \leq \underbrace{ \| \widehat \phi_{n,M,\infty} - \widehat\phi_n\|_{\HH } }_{\text{inference error}} + \underbrace{\| \widehat \phi_{n} -\phi\|_{\HH} }_{\text{approximation error}}  \lessapprox (\Delta x)^{\alpha s/(s+1)},
\end{equation}
where $\alpha$ denotes the order of the numerical integrator (see Theorem \ref{thm:rate}). That is, we achieve a rate of convergence $(\Delta x)^{\alpha s/(s+1)}$, optimal in the sense that it is the same as the numerical integrator when the kernel is smooth (i.e., $s\to \infty$).

We demonstrate the efficiency of the algorithm on three typical examples: the opinion dynamics with a piecewise linear potential (Section \ref{sec:OD}), the granular media model with a cubic potential (Section \ref{sec:cubic}), and the aggregation-diffusion with a repulsive-attractive potential (Section \ref{sec:RA}). In each of these examples, our algorithm leads to accurate estimators that can reproduce highly accurate solutions and free energy. For the cubic potential, which is smooth, our estimator achieves the optimal rate of convergence. For non-smooth piecewise linear potential and the singular repulsive-attractive potential, our estimator converges at sub-optimal rates.  

% Mean-field equations are widely used in the above areas and in machine learning. XXX. 

The remainder of the paper is organized as follows. We present the learning algorithm in Section \ref{sec:InferAlg}, where we introduce the error functional and the estimator, discuss the choice of basis functions for the hypothesis space, provide practical guidance on regularization and selection of optimal dimension.  Section \ref{sec:Conv} studies the rate of convergence of the estimator when the space mesh size decreases, with the technical proofs postponed in Appendix \ref{appendixA}. Numerical examples in Sections \ref{sec:num} demonstrate the accuracy of our algorithm on three examples with different types of kernels.  We discuss the limitations this study and directions for future research in Section \ref{sec:conclusion}.

\paragraph{Notation} 
We will use the notations in Table \ref{tab:notation}. 
% For any  $\psi: \R^+\to \R$, we denote
% \begin{equation*}\label{eq:K_psi}
% K_\psi (x) = \psi(|x|)\frac{x}{|x|},  \ \Psi' = \psi. 
% \end{equation*}
We denote by $\|\cdot\|_{\infty}$ and  $\|\cdot\|_{k,\infty}$ the $L^\infty$ norm and the $W^{k,\infty}$ norm, respectively, on the corresponding domains. For example,    $\|u\|_{\infty}$ and  $\|u\|_{1,\infty}$  denote the $L^\infty(\Omega\times [0,T])$ and $W^{1,\infty}(\Omega\times [0,T])$ norms, 
\begin{align*}
\norm{u}_\infty  = \sup_{x\in \Omega, t\in [0,T]} |u(x,t)|,\quad   \norm{u}_\WOT =  \norm{\nabla_{x,t}u}_\LOT + \norm{u}_\LOT. 
\end{align*}
Similarly,  $\|\phi\|_{\infty}$ and    $\|\phi\|_{k,\infty}$ denote  the $L^\infty(\supp{\rhoT})$ and $W^{k,\infty}(\supp{\rhoT} )$ norms, respectively. 
\begin{table}[!t] 
{\small
	\begin{center} 
		\caption{ \, Notations} \label{tab:notation}
		\begin{tabular}{ l  l }
		\toprule % \hline
			Notation   &  Description \\  \hline
	                $\phi $ and $\Phi$    & true interaction kernel and potential, $\phi= \Phi'$\\
	                $\psi$ and $\Psi'$     & a generic interaction kernel and potential, $\psi= \Psi'$\\
	                $K_\psi (x)= \psi(|x|)\frac{x}{|x|} $    & interaction kernel with kernel $\psi$       \\
	                $x\cdot y $   & the inner product between vector $x,y$ \\ 
%	                $\nabla_{x,t} f$  & the gradient of $f(x,t)$ with respect to $(x,t)$ \\
	                $\|\cdot\|_{\infty}$ and  $\|\cdot\|_{k,\infty}$  & the $L^\infty$ norm and $W^{k,\infty}$ norms, $k\geq 1$\\
	                 $L^2(\rhoT)$                                      & L2 space with $\rhoT$ in \eqref{eq:rhoavg} \\
%	                  $\rkhsH$                                           & the RKHS with reproducing kernel $\GbarT$ in \eqref{eq:GbarT} \\
%	                  $\HH = L^2(\rhoT)$ or $\rkhsH$                        & the function spaces of learning  \\             
                        $\mH= \mathrm{span}\crl{\phi_i}_{i=1}^n$ & hypothesis space with basis functions $\phi_i$ \\
                        $\mE_{M,L}(\psi)$   & error functional in \eqref{eq:costFn_data}, from data $\{ u(x_m,t_l)\}_{m,l=0}^{M,L}$ \\
                        $\widehat \phi_{n,M,L}$ & estimator: minimizer of $\mE_{M,L}$ on $\mH$ \\ 
			\bottomrule	
		\end{tabular}  
	\end{center}
}
\end{table}

\section{Inference of the interaction kernel}\label{sec:InferAlg}
We introduce an efficient scalable algorithm estimating the interaction kernel by least squares in a nonparametric fashion. The key is a probabilistic error functional, which is the expectation of the negative likelihood ratio of the diffusion process described by the mean-field equation. Our estimator, the minimizer of the error functional, is then an extension of the maximal likelihood estimator (MLE). Remarkably, we can compute the estimator and the error functional without using any spatial derivative of the solution, allowing us to recover the interaction kernel from sparse data. We also discuss the function space of learning, the choice of basis functions and selection of dimension for the hypothesis space, and regularization.  

\subsection{The error functional and estimator}

% We introduce an error functional that is the expectation of the negative likelihood ratio of the diffusion process corresponding to equation \eqref{eq:MFE}. As we show below, the likelihood ratio leads to a error functional that is a quadratic functional of the interaction kernel, so that the MLE can be computed simply by least squares. Furthermore, the error functional can be computed from the data consisting of the solution $u$. In particular, with integration by parts, we can compute the estimator and the error functional without using the second order derivatives of $u$, allowing for the recovery of the interaction kernel from sparse noisy data. 

Suppose first that the data is a continuous space-time solution $u$ on $[0,T]$, we derive an error functional from the likelihood of the diffusion process $(\barX_t, t\in [0,T])$ described by the mean-field equation. More precisely, Eq.\eqref{eq:MFE} is the Fokker-Planck equation (or the Kolmogorov forward equation) of the following nonlinear stochastic differential equation 
\begin{equation} \label{eq:nSDE}
\left\{
\begin{aligned}
d\barX_t =& -K_\phi*u(\barX_t,t)dt + \sqrt{2\nu }dB_t,\\
\mathcal{L}(\barX_t) = &\, u(\cdot,t),
\end{aligned}
\right.
\end{equation}
for $t\geq 0$. Here $\mathcal{L}(\barX_t)$ denotes the probability density of $\barX_t$ if $u$ is a regular solution, or the probability measure of $\barX_t$ if  $u$ is a distribution solution, depending on the initial condition and the admissible set of the interaction kernel. In either case, we can write the convolution as 
\[
K_\phi*u(\barX_t,t) = \E[K_\phi(\barX_t-\barX_t')\mid \barX_t] = \E[\phi(\abs{\barX_t-\barX_t'}) \frac{\barX_t-\barX_t'}{\abs{\barX_t-\barX_t'}} \mid \barX_t],
\]
where $\barX_t'$ is an independent copy of $\barX_t$. 

We start from the ambient function space for the interaction kernel: $L^2(\rhoT)$, where $\rhoT$ is the average-in-time distribution of $|\barX_t'-\barX_t|$ (denoted by $\rho_t$) on $[0,T]$:  
 \begin{equation} \label{eq:rhoavg} 
  \rhoT(dr) : = \frac{1}{T} \int_0^T \rho_t(dr) dt, \quad \rho_t(dr) : =  \E[\delta(|\barX_t'-\barX_t|\in dr)].  
 \end{equation}
Note that $\rhoT$ depends on the initial distribution $u(\cdot,0)$ and the true interaction kernel $\phi$. We point out that the measure $\rhoT$ is different from the empirical measure of pairwise distances in particle systems \cite{LZTM19,LMT20,LLMTZ19}, because $\barX_t$ and $\barX_t'$ are independent copies and are no longer interacting particles. However, in view of inference, the high probability region of $\rhoT$ is where $\abs{\barX_t-\barX_t'}$ explores the interaction kernel the most, as such, the natural function space of inference is  $L^2(\rhoT)$. Also, the space $L^2(\rhoT)$ ensures that our error functional below is well-defined. % The error functional in \eqref{eq:costFn} and the regression vector in \eqref{b_best} avoids the use of $\Delta u$ and it follows from integration by parts.  Note that all these integration by parts require that $u$ and $\grad u$ are both vanishing for $x\in \partial \Omega$.

\begin{theorem}[Error functional]\label{thm:costFn}
Let $u$ be a solution to \eqref{eq:MFE} on [0,T] with interaction kernel $\phi$.  Let $\psi \in L^2(\rhoT) $ with $\rhoT$ in \eqref{eq:rhoavg}, $\Psi(r)=\int_0^r \psi(s)ds$ and $K_\psi(x) = \grad \Psi(|x|)$. The error functional 
\begin{align}\label{eq:costFn}
  \mE(\psi)   : = & \frac{1}{T}\int_0^T \int_{\R^d} \left[  \abs{K_\psi*u}^2 u 
+  2\partial_t u (\Psi *u)    + 2\nu \grad u \cdot (K_{\psi} *u) \right ] dx\ dt  
\end{align}
is the expectation of the average-in time negative log-likelihood of the process $\barX_t$ in \eqref{eq:nSDE}. Furthermore, if $\psi\in W^{1,\infty}$, we can replace the integrand $\grad u \cdot (K_{\psi} *u)$ by $-u (\Delta \Psi *u)$. 
\end{theorem}

\begin{proof} We denote by $\mathcal{P}_\phi$ the law of the process $(\barX_t)$ on the path space with initial condition $\barX_0\sim u(\cdot,0)$, with the convention that $\mathcal{P}_0$ denotes the Winner measure. Then, the negative log-likelihood ratio of a trajectory $\barX_{[0,T]}$ from $\mathcal{P}_\psi$ relative to $\mathcal{P}_\phi$ is  
(see e.g., \cite[Section 1.1.4]{Kut04} or \cite[Section 3.5]{KS98}) 
\begin{align}\label{lkhd_nSDE}
 \mathcal{E}_{\barX_{[0,T]}}(\psi)= - \log \frac{d \mathcal{P}_\psi }{d\mathcal{P}_\phi}(\barX_{[0,T]}) =  \frac{1}{T}\int_{0}^T \left(| K_\psi *u (\barX_t)|^2 dt - 2 \langle K_\psi*u (\barX_t), d\barX_t \rangle \right),  
\end{align}
where $\frac{d \mathcal{P}_\psi }{d\mathcal{P}_\phi}$ is the Radon-Nikodym derivative. 

Taking expectation and noting that $d\barX_t =  K_\phi*u(\barX_t)dt + \sqrt{2\nu }dB_t$,   
we obtain
\begin{align}\label{lkhd_nSDE_expn}
\E \mathcal{E}_{\barX_{[0,T]}}(\psi)=& \frac{1}{T}\int_{0}^T \E \left[ |K_\psi *u  (\barX_t)|^2 - 2 \innerp{ K_\psi*u (\barX_t)} {K_\phi*u (\barX_t) } \right] dt.  
\notag \\ = & \frac{1}{T}\int_0^T \int_{\R^d}  \left[ \abs{K_\psi*u}^2 u -2  u \innerp{( K_\phi*u)}{ ( K_\psi*u)} \right] \ dx\ dt
\end{align}
in which we used that the fact that for any $\psi,\phi \in L^2(\rhoT)$  (recalling that $u(\cdot,t)$ is the law of $\barX_t$), 
\begin{align*}
\E[  \innerp{K_\psi*u(\barX_t)}{K_\phi*u(\barX_t)} ]  =&  \int_{\R^d}   u \innerp{( K_\phi*u)}{ ( K_\psi*u)}  \ dx. 
\end{align*}

Noticing that $K_\psi *u  = \nabla \Psi*u = \nabla (\Psi *u) $ for any  $\psi =\Psi'$ and using \eqref{eq:MFE}, we have
\begin{align}
 & \int_0^T \int_{\R^d} u \innerp{(K_\phi*u)}{  (K_\psi*u)} dx\ dt\nonumber=  \int_0^T \int_{\R^d} u \innerp{(K_\phi*u)}{ ( \nabla \Psi *u) } dx\ dt \notag\\
 = & - \int_0^T \int_{\R^d} (\nabla\cdot [ u (K_\phi*u)] ) \Psi *u dx\ dt \quad (\text{ by integration by parts})  \notag \\
  = & -\int_0^T \int_{\R^d} (\partial_t u -\nu  \Delta u) (\Psi *u) dx\ dt \quad (\text{ by Equation  \eqref{eq:MFE}}) \notag \\
  =& -\int_0^T \int_{\R^d} \partial_t u (\Psi *u)  dx\ dt  
  - \nu \int_0^T \int_{\R^d} \nabla u \cdot (K_{\psi} *u) dx\ dt, \, (\text{integration by parts}). \label{eq:IBP}
  \end{align}

Combining \eqref{lkhd_nSDE_expn}, \eqref{eq:IBP} and \eqref{eq:costFn}, we obtain 
\begin{equation}\label{eq:costFn_phi}
\E \mathcal{E}_{\barX_{[0,T]}}(\psi) =  \frac{1}{T}\int_0^T \int_{\R^d}  \left[ \abs{K_\psi*u}^2 u -2  u \innerp{( K_\phi*u)}{ ( K_\psi*u)} \right] \ dx\ dt 
= \mE(\psi). 
\end{equation}

At last, from integration by parts, we can replace the integrand $\grad u \cdot (K_{\psi} *u)$ by $-u (\Delta \Psi *u)$ if $\psi \in W^{1, \infty}$.
\end{proof}

  To simplify the notation, we introduce the following bilinear form: for any $\phi, \psi \in \mathcal{H}$,  
  \begin{align}
  \rkhs{\phi, \psi}_{\GbarT}
: = & \frac{1}{T}\int_0^T \int_{\R^d} \innerp{(K_{\phi}*u)}{ (K_{\psi}*u)}   u(x,t)dx\ dt \nonumber  \\
% = & \frac{1}{T} \int_0^T \int_{\R^d} \int_{\R^d} K_{\phi_i}(y)u(x-y, t) dy  \cdot  \int_{\R^d} K_{\phi_j}(z)u(x-z, t) dz  u(x,t)dx\ dt \nonumber \\
= & \frac{1}{T} \int_0^T \int_{\R^d} \int_{\R^d} K_{\phi}(y)  \cdot
K_{\psi}(z) \int_{\R^d} u(x-y, t) u(x-z, t)   u(x,t) dx dydz \ dt \nonumber \\  
= &\int_{\R^+} \int_{\R^+} \phi(r)\psi(s)  \ \overline G_T(r,s)  \ dr\ ds,  \label{eq:bilinearF}
\end{align}
where the kernel $\GbarT(r,s)$, obtained by a change of variables to polar coordinates, is  
\begin{align}\label{eq:GbarT}
\GbarT(r,s) =\frac{1}{T} \int_0^T \int_{\mathbb{S}^{d-1}} \int_{\mathbb{S}^{d-1}} \int_{\R^d} \innerp{\xi}{\eta} \,  (rs)^{d-1}  u(x-r\xi,t)u(x-s\eta, t)u(x,t) dxd\xi d\eta \ dt.
\end{align}
Here  $\mathbb{S}^{d-1}$ denotes the unit sphere in $\R^d$. We prove in \cite{LangLuWang20} that $\GbarT$ is symmetric positive semi-positive definite. Thus, $ \rkhs{\cdot,\cdot}_{\GbarT}$ is the inner product of the reproducing kernel Hilbert space (RKHS), denoted by $H_{\GbarT}$ with $\GbarT$ as reproducing kernel.

Then, we can write the error functional, in view of \eqref{eq:costFn_phi}, as 
\begin{align}\label{eq:cosrFn_bilinear}
  \mE(\psi) = \rkhs{\psi, \psi}_{\GbarT}  -2\rkhs{\psi, \phi}_{\GbarT}= \rkhs{\psi-\phi,\psi-\phi}_{\GbarT} -\rkhs{\phi,\phi}_{\GbarT}. 
\end{align}
That is, the error functional is the square of the RKHS norm of $\psi-\phi$, minus a constant that is the square of the RKHS norm of the true kernel. Thus, the convergence of the error functional leads to convergence of the estimator.  Furthermore, it is quadratic, so we can compute its minimizer on a finite-dimensional hypothesis space by least squares. 

\begin{theorem}[Estimator]\label{thm:MLE}
For any space $\mH = \mathrm{span}\crl{\phi_i}_{i=1}^n \subset L^2(\rhoT)$ such that the normal matrix $A$ in \eqref{Amat} is invertible, the unique minimizer of the error functional $\mE$ on $\mH$ is given by 
\begin{equation} \label{eq:mle_Hn}
\widehat\phi_n = \sum_{i = 1}^n \widehat{c_i} \phi_i,  \text{ with } \widehat{c} = A^{-1} b.
\end{equation}
where the normal matrix $A$ and vector $b$ are given by
\begin{align}
  A_{ij} &= \rkhs{\phi_i, \phi_j}_{\GbarT} =  \frac{1}{T}  \int_0^T \int_{\R^d} 
  (K_{\phi_i}*u)\cdot (K_{\phi_j}*u)   u(x,t)dx\ dt,  \label{Amat} \\
  b_i &= \rkhs{\phi, \phi_i}_{\GbarT} % =\frac{1}{T} \int_0^T \int_{\R^d} (K_\phi *u) (K_{\phi_i} *u) u(x,t)dx\ dt\nonumber\\
  = -\frac{1}{T}  \int_0^T \int_{\R^d} \left[ \partial_t u \, \Phi_i *u  - \nu \grad u \cdot ( K_{\phi_i} *u)\right] dx\ dt, \label{b_best} 
\end{align}
where $\Phi_i(r) = \int_0^r \phi_i(s)ds$. Again, we can replace $-\grad u \cdot ( K_{\phi_i} *u)$  by $u ((\grad\cdot K_{\phi_i}) *u)$ if $\phi_i\in W^{1,\infty}$. 
\end{theorem}

\begin{remark} \label{rmk:Phi(0)}
We set $\Phi_i(0)=0$ for the anti-derivative of $\phi_i$ for simplicity. In general, the constant $\Phi_i(0)$ does not affect the integral $ \int_0^T \int_{\R^d} \partial_t u \, \Phi_i *u dx dt $ in \eqref{b_best}, because  for any constant $c$, we have $c*u =c$ and 
$
\int_0^T \int_{\R^d}  \partial_t u c dxdt =c  \int_{\R^d} u(x, 0) - u(x, T) dt = c-c = 0
$. 
\end{remark}

\begin{proof}
Recall that with the bilinear form \eqref{eq:bilinearF}, the error functional can be written as \eqref{eq:cosrFn_bilinear}. Denote for each $\psi\in \mH$ by
$\psi = \sum_{i = 1}^n c_i \phi_i$.
We can now write the error functional on $ \mH$ as, 
\begin{align}\label{errfcnl_Ab}
\mE(\psi) = \mE(c)&= c^\top A c - 2b^\top c,
\end{align}
where $A$ is given by \eqref{Amat} and $b$ are given by
\begin{align*}
  % A_{ij} &= \rkhs{\phi_i, \phi_j}_{\GbarT} = \frac{1}{T}  \int_0^T \int_{\R^d}  (K_{\phi_i}*u) (K_{\phi_j}*u)   u(x,t)dx\ dt,\\
  b_i &= \rkhs{\phi, \phi_i}_{\GbarT} = \frac{1}{T} \int_0^T \int_{\R^d} (K_\phi *u)\cdot (K_{\phi_i} *u) u(x,t)dx\ dt \notag\\
  &= -\frac{1}{T} \int_0^T \int_{\R^d} (\partial_t u - \nu  \Delta u) (\Phi_i *u) dx\ dt\notag \\
  &= -\frac{1}{T}  \int_0^T \int_{\R^d} \partial_t u (\Phi_i *u)  dx\ dt
  -\nu  \frac{1}{T} \int_0^T \int_{\R^d} \nabla u \cdot (K_{\phi_i} *u) dx\ dt ,
\end{align*}
where in the last equality, we used integration by parts to get rid of $\nu \Delta u$. Applying integration by part again to $ \int_{\R^d} \nabla u \cdot (K_{\phi_i} *u)$ and note that $\nabla (K_{\phi_i} *u) = (\grad\cdot K_{\phi_i}) *u$, we obtain \eqref{b_best}. 

Since the error functional is quadratic, the minimizer can be given explicitly by \eqref{eq:mle_Hn}. 
\end{proof}

\begin{remark} 
The normal matrix's invertibility depends on the basis functions $\crl{\phi_i}_{i=1}^n$ of $\hypspace$. It is the identity matrix when the basis functions are orthonormal in the RKHS of $\GbarT$. When the basis functions are orthonormal in $L^2(\rhoT)$, its smallest eigenvalue is the smallest eigenvalue of the integral operator with kernel $\GbarT$ on $\hypspace$ (see Proposition {\rm\ref{prop:A_coercivity}}). Thus, to make the normal matrix invertible, we need a coercivity condition in Definition {\rm\ref{def_coercivty}}. 
% In practice, even when the bilinear form is uniformly bounded below on $\mH$, the regression matrix may be almost singular with a large conditional number when its dimension is hight. In this case, a regularization is recommended, see Section XXX for more discussions. 
\end{remark}
\begin{remark}[The PDE discrepancy error functional] Our error functional in \eqref{eq:costFn} has two advantages over the PDE discrepancy error functional % measuring the discrepancy between both sides of the PDE: 
\begin{equation*}
\mE_0(\psi) %&= \norm{\nabla.(u [(K_\psi- K_\phi)*u]}^2_{L^2\bigp{[0,T], \ \R^d}}\\
% &=  \norm{\nabla.(u ( K_\psi*u)) - g}^2_{\Ltrd}\\  &
= \int_0^T \int_{\R^d} \abs{\nabla.(u(K_\psi*u)) - g}^2 dx\ dt.
\end{equation*}
where $g = \partial_t u  - \nu \Delta u $. %= \nabla.(u(K_\phi*u))$. 
First, it requires the derivatives $\grad u$ and $\Delta u$, because the integration by parts does not apply. Second,  our error functional makes sufficient use that $u(\cdot,t)$ is a probability density so that its components are expectations, allowing for Monte Carlo approximations. 
%=======
%% &=  \norm{\nabla.(u ( K_\psi*u)) - g}^2_{\Ltrd}\\
%&= \int_0^T \int_{\R^d} \abs{\nabla.(u(K_\psi*u)) - g}^2 dx\ dt.
%\end{align*}
%where $g(x,t) = \partial_t u(x,t) -  v\Delta u(x,t) = (K_\phi*u)(x,t)$. We point out it requires the derivatives $\grad u$ and $\Delta u$, because the integration by parts does not apply and it does not take advantage that $u(\cdot,t)$ is a probability density.  
%>>>>>>> Stashed changes
\end{remark}

\paragraph{Estimator from discrete data} 
When data are discrete in space-time, we approximate the integrals in the estimator and the error functional by numerical integrators. For simplicity, we consider only data on a regular mesh for $d=1$ and use the Euler scheme. In practice, we could use higher-order numerical methods for the integration and convolution, for instance, the trapezoid method for the integrals and Fourier transform for the convolution. Note also that these integrations are expectations, so in general, particularly for high dimensional cases, the data can also be independent samples of the distribution, and we approximate the integrations by the empirical mean. 

Suppose that the data are 
$ \{ u(x_m,t_l)\}_{m,l=1}^{M,L}, \text{ with }  t_l = l \Delta t$ for $L= T/\Delta t$ and with $\{x_m\}_{m=1}^M$ being a uniform mesh of $ \Omega$ with length/area $\Delta x$.

From these data, we approximate all the integrals by Euler scheme. We approximate $\rhoT$ in \eqref{eq:rhoavg} by its the empirical measure:
\begin{align}\label{rhoAppr} 
\rho_L^M(dr) = \frac{1}{L}\sum_{l=1}^L  \sum_{m,m'=1}^{M,M}u(x_m,t_l) u(x_{m'},t_l)  \delta_{|x_m-x_{m'}|}(r)dr. 
\end{align}
With  basis functions $\{\phi_i\}$ and $\psi= \sum_{i=1}^n c_i\phi_i$, we approximate the error functional in \eqref{errfcnl_Ab} by  
\begin{align}\label{eq:costFn_data}
\mE_{M,L}(\psi)= \mE_{M,L} (c)=  c^\top A_{M,L} c - 2b_{M,L}^\top c,
\end{align}
where the normal matrix $ A_{M,L}$ and vector $ b_{M,L}$, approximating $A$ and $b$ in \eqref{Amat}-\eqref{b_best}, are 
\begin{align}
  A_{n, M,L}^{i,j} & := \frac{1}{L} \sum_{l=1}^L \sum_{m=1}^M  \left[\left( P_{n, M, L}^i \cdot P_{n, M, L}^j  \right) u \right](x_m,t_l) \Delta x , \label{AnML}\\
  b_{n, M,L}^{i} & :=    -  \frac{1}{L} \sum_{l=1}^L \sum_{m=1}^M  \left[  \widehat{\partial_t u} \ Q_{n, M, L}^i  + \nu u R_{n, M, L}^i \right] (x_m,t_l)  \Delta x.\label{bnML}
\end{align}
Here $P_{n, M, L}^i$, $Q_{n, M, L}^i$ and $R_{n, M, L}^i$ are Riemann sum approximations of $K_{\phi_i}*u$, $\Phi_i*u$ and $\nabla K_{\phi_i}*u$, respectively, and $  \widehat{\partial_t u}$ is the finite difference approximation of $\partial_t u$ 
\begin{equation}\label{eq:PQ}
\begin{aligned}
 % P_{n, M, L}^i(x,t) & =  I_{M,L}(K_{\phi_i}),\quad   Q_{n, M, L}^i(x,t) = I_{M,L}(\Phi_i), \quad R_{n, M, L}^i(x,t) = I_{M,L}(\nabla K_{\phi_i}), \\ 
 %  I_{M,L} (f ) &: =   \sum_{m=1}^M f(x_m)u(x-x_m, t) \Delta x, \\
 P_{n, M, L}^i(x,t) &: =   \sum_{m=1}^M K_{\phi_i}(x_m)u(x-x_m, t) \Delta x, \\
 Q_{n, M, L}^i(x,t) &: =   \sum_{m=1}^M \Phi_i(x_m)u(x-x_m, t) \Delta x, \\
  R_{n, M, L}^i(x,t) &: =   \sum_{m=1}^M \grad\cdot K_{\phi_i} (x_m)u(x-x_m, t) \Delta x ,\\
  \widehat{\partial_t u}(x,t) &:=\frac{1}{\Delta t} \sum_{l=1}^L [ u(x,t_{l})-u(x,t_{l-1})  ] \mathbf{1}_{(t_{l-1},t_{l}]}(t).
\end{aligned}
\end{equation}
A few remarks on the numerical aspects: (1) one can use high-order numerical integrators to increase the accuracy of the spatial integrals; (2) we computed $ R_{n, M, L}^i(x,t) $ assuming that the basis functions $\{\phi_i\}$ are in $W^{1,\infty}$. If $\{\phi_i\}$ are not differentiable, we can use $\grad u$ as in \eqref{b_best}; (3) in practice, we use zero padding for $u$ by setting $u(x_i,t) = 0$ if $x_i\in \partial \Omega$; (4) also, we normalize the vector $\{u(x_m,t)\}_{m=1}^M$ so that $\sum_{m=1}^M \Delta x  u(x_m,t) = 1$ for each $t$. 
This ensures that $  \sum_{m=1}^M \left[ \widehat{\partial_t u} \ Q_{n, M, L}^i \right] (x_m,t_l) $ does not depend on the constant $\Phi(0)$ in the antiderivative, as discussed in Remark \ref{rmk:Phi(0)}.

Correspondingly, the estimator is 
\begin{align}\label{eq:mle_nML}
\widehat \phi_{n,M,L} = \sum_{i = 1}^n \widehat{c}_{n,M,L}^i \phi_i,\quad \text{ with }  \widehat{c}_{n,M,L}= A_{n,M,L}^{-1} b_{n,M,L}. 
\end{align}

\subsection{Basis functions for the hypothesis space}\label{sec:basisFun}
% We assess the performance of the estimator in the following spaces:  the $L^2([0,\RadiusOmega], \rhoT)$, the Reproducing Kernel Hilbert Space of the integral kernel $\GbarT$ in \eqref{eq:GbarT}, as well as $C(\Omega)$ with the uniform norm on $\Omega$. We take two basis with eigenfunctions of $G$ and B-splines. \FL{We should avoid using $\Omega$ for the function space later. }
We consider two classes of basis function for the hypothesis space $\hypspace$: the B-spline piecewise polynomials whose knots are uniform partition of $\rhoT$'s support; 
%\QL{I set the knots to be uniformly partitioned over $\Omega$. Does it matter? \FL{It affects the A matrix. It is OK since you use regularization}}
the RKHS basis consisting of eigenfunctions of the integral operator with kernel $\GbarT$. The B-splines are universal local basis, while the RKHS basis are global basis adaptive to data.    

\paragraph{B-spline basis functions} B-spline is a class of piecewise polynomials, and is capable of representing the local information of the interaction kernel. Here we review briefly the recurrence definition and properties of the B-splines, for more details we refer to \cite[Chapter 2]{piegl1997_NURBSBook} and \cite{lyche2018_FoundationsSpline}.    

Given a nondecreasing sequence of real numbers $\{r_0,r_1,\ldots, r_m\}$ (called knots), the B-spline basis functions of degree $p$, denoted by $\crl{N_{i,p}}_{i=0}^{m-1}$, is defined recursively as
\begin{equation}\label{eq:B_spline}
\begin{aligned}
  &N_{i,0}(r) = 
  \left\{
    \begin{array}{lr}
      1,\ &r_i \leq r < r_{i+1},\\
      0,\ &otherwise,
    \end{array}
  \right.\\
  &N_{i,p}(r) = \frac{r - r_i}{r_{i + p} - r_i} N_{i, p-1}(r) + \frac{r_{i+p+1} - r}{r_{i + p + 1} - r_{i + 1}}N_{i + 1, p- 1}(r).
\end{aligned}
\end{equation}
Each B-spline basis function $N_{i,p}$ is a nonnegative piecewise polynomial of degree $k$, locally supported on $[r_i,r_{i+p+1}]$, and it is $p-k$ times continuously differentiable at a knot, where $k$ is the multiplicity of the knot. Hence, continuity increases when the degree increases, and continuity decreases when knot multiplicity increases. Also, it satisfies partition unity: for each $r\in [r_i,r_{i+1}]$, $\sum_{j} N_{j,p}(s)  = \sum_{j=i-p}^i N_{j,p}(r) =1$. This basis is also called the balanced B-spline.

 For a function $f\in W^{k,\infty}$, denoting $f_\mH$ its projection to the linear space $\mH$ of B-splines with degree $p\geq k$, and denoting $D^{(k)}f$ denotes the $k$-th order derivative, we have  \cite[p.45]{lyche2018_FoundationsSpline}
\begin{equation}\label{eq:splineOrder} 
\|f-f_\mH\|_{\infty} \leq C_p h^{k} \|D^{(k)} f\|_\infty,
\end{equation}
 where $C_p$ is a constant depending on $p$, and $h=\max_{i} |r_i-r_{i-1}|$.    

We will set the knots to be the m-quantile (so that the m intervals have uniform probability) partition of the support of $\rhoT$, with augmented knots at the ends of the support interval, say $[R_{min}, R_{max}]$,  
\[
R_{min} = r_{-p+1} = \cdots = r_{0} \leq r_1 \leq \cdots \leq r_{m}
= \cdots = r_{p+m-1}= R_{min}.
\]
We set the basis functions of the hypothesis $\hypspace $, whose dimension is $m+p$, to be 
$$\phi_i(r) = N_{i-p, p}(r),\ x>0, \ i = 1,\dots, m+p. $$ 
Thus, the basis functions $\{\phi_i\}$ are degree-$p$ piecewise polynomials with knots adaptive to $\rhoT$.

\paragraph{The RKHS basis} The RKHS basis functions $\{\phi_i\}$ are the eigenfunctions of the integral operator with kernel $\frac{1}{\rhoT(s)\rhoT(r)}\GbarT(r,s)$ (defined in \eqref{eq:GbarT}) on $L^2(\rhoT)$ 
%\footnote{\FL{Quanjun, it should be  $L^2(\rhoT)$ instead of  $L^2(\R^+)$}. Because this is the function space for the error functional to be well-defined. With integral kernel above, we get $\{\phi_i\}$ orthonormal in $L^2(\rhoT)$ and $\rkhs{\phi_i, \phi_j}_{\GbarT} = \int_{\R^+} \phi_i(r)\phi_j(s) \frac{1}{\rhoT(s)\rhoT(r)} \GbarT(r,s)\rhoT(r)\rhoT(s)drds$.}
,  that is, 
\begin{equation}\label{eq:RKHS_basis}
	\int_{\R^+} \phi_i(r) \GbarT(r,s)\frac{1}{\rhoT(s)} dr = \lambda_i \phi_i(s),\ \text{ in } L^2(\rhoT).
\end{equation}

Thus, we have $\rkhs{\phi_i, \phi_j}_{\GbarT} =  \lambda_i \langle \phi_i, \phi_j\rangle_{\rhoT} = \lambda_i  \delta_{i-j}$ and it leads to a diagonal normal matrix $A$ in \eqref{Amat}.
%\begin{equation}
% \rkhs{\phi_i, \phi_j}_{\GbarT} =  \left\{  \begin{array}{lr}   \lambda_i, \ &i = j\\     0,\ &i \neq j,    \end{array}   \right.
%\end{equation}
We estimated these eigenfunctions from data. 
% by first computing the matrix of the operator on the local basis (such as the B-splines, particularly the piecewise constants basis) from data, and then sovling the eigenvectors to obtain the eigenfunctions.  
% The kernel $\GbarT$,  and its eigenfunctions are estimated from data to approximate the integrals in. 
Hence, the RKHS basis is adaptive to data. 

One can compute the these eigenfunctions by an eigen-decomposition of the matrix $( \GbarT(r_i,r_j))$ on a mesh when its size is manageable. When the mesh size is large, we can compute them 
 % We compute the eigenfunctions  efficiently without evaluating $\GbarT$ at all mesh points, but
by a linear transformation from a more convenient basis. That is, we start from linearly independent functions $\{\psi_i\}_{i=1}^n$, say piecewise polynomials, compute the eigenvectors $\{\alpha_k\in \R^n\}_{k=1}^n$ of the matrix $\widetilde A: = \rkhs{\psi_i, \psi_j}_{\GbarT}$. Then, the eigenfunctions are given by $\phi_i=\sum_{l=1}^n \alpha_{i,l} \psi_l$. 

% \FL{1.TBD. We used the integrate operator on $L^2(\R^+)$ in the above RKHS basis. A first question is $L^2$ integrability of the kernel $\GbarT$. In general, we may want to consider $L^2(R^+)$. } 
% XXX Note that by completing squares in \eqref{eq:cosrFn_bilinear} we have  $  \mE(\psi) = \rkhs{\psi-\phi,\psi-\phi}_{\GbarT} -\rkhs{\phi,\phi}_{\GbarT}$. Hence the RKHS can better describe the behavior of the error functional. In practice, we take the eigenvector of the kernel matrix with major eigenvalues. The number of major eigenvalues is related to the smoothness of the kernel $\GbarT$.

\subsection{Regularization}\label{sec:opt_reg}
In practice, the approximate normal matrix $A_{n,M,L}$ in \eqref{AnML} may be ill-conditioned or invertible, which is likely to happen when the dimension of $\hypspace$ increases because of the vanishing eigenvalues of $A$ and the numerical errors. The ill-conditioned normal matrix may amplify the numerical error in $b_{n,M,L}$ in \eqref{bnML}.% in the eigen-space of small eigenvalues. 
To avoid such an issue, we use the Tikhonov regularization (see, e.g.,\cite{hansen_LcurveIts}), which adds a norm-induced well-conditioned matrix to the normal matrix. 

% Note that the minimizing coefficient $\widehat{c}$ is given by $A^{-1}b$. In practice, there is a unavoidable computation error for $b$ caused integration by parts and taking partial derivatives. And for many cases, $A$ is ill-posed. This ill-posedness will amplify the computation error of $b$. Hence we wish to use the Tikhonov regularization. 

More precisely, we impose a regularization norm $\normmm{\cdot}$ (to be determined below) such that for any $\psi = \sum_{i = 1}^n c_i \phi_i$,  the matrix $B$ in $\normmm{\psi}^2 = c^\top B c$ is well-conditioned. We then minimize the regularized error functional (recall \eqref{errfcnl_Ab})
\begin{align*}
	\mE_\lambda(\psi) = \mE(\psi) + \lambda \normmm{\psi}^2
	% &= \frac{1}{T}\int_0^T \int_{\R^d}  \left[ \abs{K_\psi*u}^2 u -2  u \innerp{( K_\phi*u)}{ ( K_\psi*u)} \right] \ dx\ dt + \lambda \normmm{\psi}^2 \\
	% & = \rkhs{\psi, \psi}_{\GbarT}  -2\rkhs{\psi, \phi}_{\GbarT} + \lambda \normmm{\psi}^2\\
	&= c^\top(A + \lambda B)c - 2b^\top c,
\end{align*}
and the regularized estimator is 
\begin{align}\label{eq:c_reg}
	\widehat{\phi_\lambda} = \sum_{i = 1} ^ n c^i_\lambda \phi_i, \ c_\lambda = (A + \lambda B)^{-1}b.
\end{align}
We consider two regularization norms $\normmm{\cdot}$ and select the parameter $\lambda$ by the L-curve method \cite{hansen_LcurveIts}. 
\paragraph{Regularization norm}% \label{section:reg_norm}
There are many regularization norms that we can choose. For the RKHS basis, we use RKHS norm, which is the common choice \cite{cucker2002_BestChoicesa}. In this case, 
%\begin{align}\label{eq:B_Id}
$	B_{ij} = \delta_{ij}$.
%\end{align}
This is equivalent to having a prior knowledge that the coefficient $c$ is small. 
For the spline basis, we choose $\normmm{f} = \norm{f}_{H^1(\Omega)}$, and in this case 
%\begin{align}\label{eq:B_H1}
$	 B_{ij} = \rkhs{\phi_i, \phi_j}_{H^1(\Omega)}$.
%\end{align}
This is equivalent to the prior assumption that $K_\phi$ of the true interaction kernel has $H^1(\Omega)$ regularity. 

%Since the measure $U$ is concentrated in the middle of our domain, $A$ would have large constraints for functions taking value on the support of $U$. That tends to amplify the error of functions taking value out side the support of $U$. Hence we can take the norm to be ${\mathring{H}^1}$ and the regularization measure to be uniform distribution, or the compliment of $\rho$, (normalized density of $(1-\rho)$ ). 

\paragraph{Regularization parameter: L-curve}
Let $l$ be a parametrized curve in $\R^2$: 
$$l(\lambda) = (x(\lambda), y(\lambda)) := (\text{log}(\mE(\widehat{\phi_\lambda}), \text{log}(\normmm{\widehat{\phi_\lambda}})).$$
Note that $\mE(\widehat{\phi_\lambda}) = c_\lambda^\top A c_\lambda - 2b^\top c_\lambda$, and $\normmm{\widehat{\phi_\lambda}} = c_\lambda^\top Bc_\lambda$.
The optimal parameter is the maximizer of the curvature of $L$. In practice, we restrict $\lambda$ in the spectral range of $A$. Then our goal is to find 
\begin{align}\label{eq:opt_lambda}
	\lambda_{0} 
	= \argmax{\lambda_{\text{min}}(A) \leq \lambda \leq \lambda_{\text{max}} (A)}\kappa(l) 
	= \argmax{\lambda_{\text{min}}(A) \leq \lambda \leq \lambda_{\text{max}} (A)}
	\frac{x'y'' - x' y''}{(x'\,^2 + y'\,^2)^{3/2}}
\end{align}
This $\lambda_{0}$ balances the error functional $\mE$ and the regularization. We refer to \cite{hansen_LcurveIts} for more details.

\subsection{Optimal dimension of the hypothesis space}\label{sec:opt_dim}
The dimension $n$ must neither be too small nor too large to avoid under-fitting or over-fitting. Theorem \ref{thm:rate} suggests that the optimal dimension is $n\approx \, (\Delta x)^{-\alpha/(s+1)}$, where $(\Delta x)^\alpha $ is the order of convergence of the numerical integrator in the computation of the normal matrix and normal vector, and $s$ is the order of decay for the distance between the true kernel and the hypothesis space. For B-spline bases, the approximation error bound in \eqref{eq:splineOrder} suggests that $s=k$ for $\phi\in W^{k,\infty}$ when we select the degree $p\geq k$. This theoretical optimal dimension provides only an estimate on the magnitude.  In practice, to find the optimal dimension, we first select a range $[N_1,N_2]$ for the dimension;  then we choose the $n$ that minimizes the regularized error functional $\tilde{\mE_\lambda}$.

\subsection{The algorithm}
We summarize the inference of the interaction kernel in Algorithm \ref{alg:main}.

\begin{algorithm}[H]
{\small
\caption{Estimation of the interaction kernel}\label{alg:main}
\begin{algorithmic}[1]
\Require{Data $\{ u(x_m,t_l)\}$ on the  space mesh $\crl{x_m}_{m = 1}^M$ with width/area $\Delta x$ and time mesh $\crl{t_l = l\Delta t}_{l = 0}^T$.}
\Ensure{Estimated $\widehat \phi$}
\State  Estimate the empirical density $\rho_L^M$ in \eqref{rhoAppr}  and find its support $[R_{min}, R_{max}]$. 
\State Select a basis type, RKHS or B-spline, and estimate a dimension range $[N_1,N_2]$, compute the basis functions as described in Section \ref{sec:basisFun}. 
\For{$n =N_1:N_2$}
%	\If{basis type == 'RKHS'}
%		\State Compute the RKHS basis function $\crl{\phi_i}_{i = 1,\dots, n}$ by taking the leading $n$ eigenvectors of $\GbarT$.
%	\ElsIf{basis type == 'spline'}
%		\State Compute the degree $p$ B spline basis function $\crl{\phi_i}_{i = 1,\dots, n}$ using \eqref{eq:B_spline}.
%	\EndIf
	\State Compute the normal matrix and vector as in \eqref{AnML}--\eqref{eq:PQ}. 	
	% \State Compute the regularization matrix $B$ using \eqref{eq:B_Id} or \eqref{eq:B_H1}. %	\Comment{Regularization}	
	\State Determine the optimal regularization constant $\lambda_0$ by \eqref{eq:opt_lambda}.
	\State Solve $c_n$ by \eqref{eq:c_reg} and record the regularized cost $C(n) = \mE_{\lambda_0}(\sum_{i = 1}^n c^i_n \phi_i)$.
\EndFor 
\State Select the optimal dimension by $ n^* = \argmax{n \in \{ N_1,N_1+1,\dots, N_2\} } C(n)$.
\State Return the estimator $\widehat \phi = \sum_{i = 1}^{n^*} c^i_{n^*} \phi_i$.
\end{algorithmic}
}
\end{algorithm}

\section{Convergence of the estimator in mesh size}\label{sec:Conv}
We analyze the convergence of the discrete-data estimator \eqref{eq:mle_nML} to the continuous-data estimator \eqref{eq:mle_Hn} as the size of observed space-time mesh increases.

We denote by $|\Omega|$ the Lebesgue measure of $\Omega$ and denote by $\RadiusOmega$ its radius.  Note that support of $\rhoT$ is  $\supp{\rhoT}\subset [0,\RadiusOmega]$.  For simplicity of notation, we consider only the case when $\Omega =[a,b] \in \R^1$ and the generalization to a higher dimension is immediate. We assume that the data are
\begin{equation} \label{eq:meshData1D} 
\textbf{Data: } \quad 
\begin{aligned}
 & \{ u(x_m,t_l), m=1,\ldots, M; l= 1,\ldots,L\},  \text{ with }  \\
& x_m= a+m\Delta x, t_l = l \Delta t, \quad M = (b-a)/{\Delta x},  L= T/\Delta t. 
\end{aligned}
\end{equation}

In this section, we make the following assumptions on $u$ and the basis functions of $\hypspace$. 
\begin{assumption}[Hypothesis space] \label{assumption_basis}
Assume that the basis functions of $\hypspace=\mathrm{span} \{\phi_i\}_{i=1}^n $ satisfy
\begin{itemize}
\item orthonormal in $\HH$, which is either $L^2(\rhoT)$ or the RKHS with reproducing kernel $\GbarT$ in \eqref{eq:GbarT}. 
\item uniformly bounded in $W^{2,\infty}([0,\RadiusOmega])$ with notations
\begin{equation}
\begin{aligned}
c^\infty_\hypspace :=  \max_{1\leq i\leq n} \|\phi_i\|_\infty, \quad
c^{1,\infty}_\hypspace := \max_{1\leq i\leq n} \|\phi_i\|_{1,\infty}, \quad
c^{2,\infty}_\hypspace :=  \max_{1\leq i\leq n} \|\phi_i\|_{2,\infty}< \infty . 
%\norm{\Phi_i}_{L^\infty(\Omega)} +  \norm{\phi_i}_{L^\infty(\Omega)} +  \norm{\phi_i'}_{L^\infty(\Omega)}
\end{aligned}
\end{equation}
\end{itemize}
\end{assumption}

\begin{assumption} \label{assumption_uxt}
Assume that solution $u \in W^{2,\infty}(\Omega\times [0,T])$ satisfies $\|u\|_\WBOT< \infty$. 
\end{assumption}

%The regularity of the kernel and the solution enables us to control the numerical error in the numerical integration of the integrals in the error functional. But they are not necessary for the inference: without these regularity conditions, the estimator may converge at a slower rate due the numerical integration error. 

We remark that the second order derivatives of the solution are necessary to control of the Riemann sum approximation of the integrals. With stronger regularity on the solution and higher-order approximations of the integrals than the Euler scheme, one can obtain higher order convergence in space and time. In the other direction, since these integrals are expectations, they can be approximated Monte Carlo, we expect to remove these regularity assumptions in forthcoming research. 

To make the estimators in \eqref{eq:mle_nML} and  \eqref{eq:mle_Hn} well-defined, the normal matrices must be invertible. When the basis functions in Assumption \ref{assumption_basis} are orthonormal in the RKHS $\rkhsH$, the normal matrix \eqref{Amat} is the identity matrix and thus invertible. When these basis functions are orthonormal in $L^2(\rhoT)$, we need the following coercivity condition. It extends of the coercivity condition for $N$-particle systems defined in \cite{LLMTZ19,LMT20,LZTM19}. 
%%%%%%%%%% ========= definition of coercivity condition 
\begin{definition}[Coercivity condition] \label{def_coercivty}
The  system \eqref{eq:MFE} on $[0,T]$ % with an interaction kernel $\phi$
satisfies a coercivity condition on a finite-dimensional linear subspace $\mathcal{H}\subset L^2(\bar \rho_T)$ with $\bar \rho_T$ defined in \eqref{eq:rhoavg} if
\begin{equation} \label{eq:coercivityDef}
c_{\mathcal{H},T} : = \quad \inf_{h\in  \mathcal{H}, \, \|h\|_{L^2(\wbar \rho_T)}=1}   \rkhs{h, h}_{\GbarT} >0,
\end{equation}
where $\rkhs{\cdot, \cdot}_{\GbarT} $ is defined in \eqref{eq:bilinearF}. When $\hypspace\subseteq L^2(\wbar \rho_T)$ is infinite-dimensional, we say  coercivity condition holds on $\hypspace$ if it holds on each of $\hypspace$'s finite dimensional linear subspace.
\end{definition}

We show that the coercivity constant is the minimum eigenvalue of the normal matrix $A$ in \eqref{Amat}. 
\begin{proposition}\label{prop:A_coercivity}
  Suppose the coercivity condition holds on the space $\mathcal{H}=\mathrm{span}\crl{\phi_i}_{i=1}^n \subset L^2(\rhoT)$. Let $A$ be the normal matrix in \eqref{Amat}, then the smallest singular value of $A$ is $\lambda _{min}(A) = c_{\mathcal{H}, T}$
\end{proposition}
\begin{proof}
  For an arbitrary $\psi\in \mathcal{H}$, we can write $\psi = \sum_{i = 1}^n c_i \phi_i $. Hence by the coercivity condition, 
  $$
  c^T A c = \rkhs{\psi, \psi}_{\GbarT}
  \geq 
  c_{\mathcal{H}, T} \norm{c}^2.
  $$
  Note that the space $\mH$ here is finite dimensional, the supreme in the definition of coercivity constant is attained by some $\psi^*$, hence we have $\lambda_{min}(A) = c_{\mH, T}$.
\end{proof}

\subsection{Error bounds for the estimator}\label{sec:errBdMLE}
We show that the error of the estimator $ \widehat \phi_{n,M,L}$ in \eqref{eq:mle_nML} converges as $M\to \infty$ and $L\to \infty$. Here $\HH$ is either $L^2(\rhoT)$ or the RKHS with reproducing kernel $\GbarT$ in  \eqref{eq:GbarT}. 

\begin{theorem}[Error bounds for the estimator] \label{thm:error_discreteTime}
Let the hypothesis space $\hypspace=\mathrm{span} \{\phi_i\}_{i=1}^n  $ satisfy Assumption \rmref{assumption_basis} and denote $\widehat \phi_n$ the projection of $\phi$ on $\hypspace\subset \HH$. Suppose that the coercivity condition holds on $\hypspace$ with a constant $c_{\mH,T}>0$ if $\HH= L^2(\rhoT)$, or set $c_{\mH,T}=1$ if $\HH$ is the RKHS. Then, the error of the estimator $\widehat \phi_{n,M,L}$ in \eqref{eq:mle_nML} satisfies
\begin{align}\label{eq:numErr}
\| \widehat \phi_{n,M,L} - \widehat \phi_n \|_{\HH}  \leq 2{c_{\mH,T}}^{-1}\left(c^b \sqrt{n} + c^A n\norm{\phi}_{\HH} \right) (\Delta x +  \Delta t ) ,
\end{align}
where  $c^A := 2|\Omega|(1+|\Omega|)(c^{1,\infty}_\hypspace)^2 \norm{u}_\WOT^2  $ and $c^b:=3   |\Omega|(1+\RadiusOmega+\nu) c^{2,\infty}_\hypspace (\norm{u}_\WOT^2+ \norm{u}_\WBOT)$. 
%\begin{equation} \label{eq:constants_cAb}
% \begin{aligned}
% c^A &:= 2|\Omega|(1+|\Omega|) (c^{1,\infty}_\hypspace)^2 \norm{u}_\WOT^2 , \\
% c^b &:=  3 |\Omega|(1+\RadiusOmega+\nu) c^{2,\infty}_\hypspace (\norm{u}_\WOT^2+ \norm{u}_\WBOT). 
% \end{aligned}
% \end{equation}  
\end{theorem}

\begin{proof}[Proof of Theorem \ref{thm:error_discreteTime}]
Notice that $\widehat \phi_{n,M,L}$ and $\widehat \phi_n$ are given by 
$$\widehat \phi_{n,M,L} = 
\sum_{i = 1}^n \widehat c_{n, M, L}^{\,i} \phi_i, \ \
\widehat \phi_n = 
\sum_{i = 1}^n \widehat{c}^{\,i} \phi_i,
$$
where $\widehat c_{n, M, L} = A_{n, M, L}^{-1} b_{n, M, L}$ and $\widehat c = A^{-1} b$, we have
\begin{align*}
  \| \widehat \phi_{n,M,L} - \widehat \phi_n \|
  _{\HH} 
 % =   \norm{\sum_{i = 1}^n (\widehat c_{n, M, L}^i - \widehat c^i)\phi_i}_{\HH} = 
  = \norm{\widehat c_{n, M, L} - \widehat c} = 
  \norm{A_{n, M, L}^{-1} b_{n, M, L} - A^{-1}b}
\end{align*}
Also, by the formula $A_{n, M, L}^{-1} - A^{-1} = 
A_{n, M, L}^{-1} (A- A_{n, M, L} ) A^{-1}$, we have
\begin{align*} 
  &\norm{A_{n, M, L}^{-1} b_{n, M, L} - A^{-1}b}
  = \norm{  A_{n, M, L}^{-1}(b_{n, M, L} - b) +   (A_{n, M, L}^{-1} - A^{-1})b}\\
%  = &\norm{    A_{n, M, L}^{-1}\left[ (b_{n, M, L} - b) +     (A - A_{n, M, L})A^{-1}b\right]}\\
  \leq 
  &\norm{A_{n, M, L}^{-1}}\left(
  \norm{b_{n, M, L} - b} + \norm{A_{n, M, L} - A}\cdot\norm{A^{-1}b}
  \right).
\end{align*}
By Proposition \ref{Aberror}, for small enough $\Delta x$ and $\Delta t$, we have $\norm{A - A_{n, M, L}} \leq \frac{c_{\mH,T}}{2}$. Hence $\norm{A_{n, M, L}^{-1}}\leq 2{c_{\mH,T}}^{-1}$. Note that $\norm{A^{-1}b}= \norm{\widehat \phi_n}_{\HH}\leq \norm{\phi}_{\HH}$ is independent of the mesh size. Thus, we obtain  \eqref{eq:numErr} by \eqref{eq:Ab_error} in Proposition \ref{Aberror}. 
\end{proof}

\begin{remark}[High order numerical integrators]\label{rmk:high-order}
For a fixed $n$, the rate of convergence is the same as the order of the numerical integrators in the computation of $A$ and $b$. Hence, when the solution and the kernel are smooth, we can achieve faster convergence by using a high-order numerical integrator. 
\end{remark}

The next theorem shows that with an optimal dimension for the hypothesis space, the estimator from continuous-time observations converges at a  rate optimal in the sense that it is almost the same as the order of the numerical integrator when the true kernel is smooth. 

\begin{theorem}[Optimal rate of convergence]\label{thm:rate} 
Assume $\Delta t =0$ and consider the estimator  $\widehat \phi_{n,M,\infty}$ in \eqref{eq:mle_nML} on $\hypspace$ with dimension $n$. Assume that as $M=\abs{\Omega}/(\Delta x)\to\infty$, we have  
$\|\widehat \phi_{n,M,\infty}- \widehat \phi_n \|  _{\HH}\lessapprox n (\Delta x)^{\alpha}$, 
which is true with $\alpha=1$ in Theorem \ref{thm:error_discreteTime}, 
 and $\| \widehat \phi_n - \phi\|_{\HH} \lessapprox n^{-s} $ with $s\geq 1$ when $n$ increases.  Then, with an optimal dimension $n\approx \, (\Delta x)^{-\alpha/(s+1)}$, we can achieve the rate 
 \[\|\widehat \phi_{n,M,\infty}- \phi\|  _{\HH} \lessapprox  (\Delta x)^{\alpha s/(s+1)}.
 \]  
 Furthermore,  the empirical error functional $\mE_{M,\infty}$ converges at the rate $2\alpha s/(s+1)$. 
\end{theorem}
\begin{proof}
 Note that the total error in the estimator consists of inference error and approximation error: 
 \[
 \|\widehat \phi_{n,M,\infty}- \phi\|  _{\HH} \leq  \|\widehat \phi_{n,M,\infty}- \widehat \phi_n \|  _{\HH}  + \|\widehat \phi_n -\phi\|  _{\HH}. 
 \]
% Suppose that $\| \widehat \phi_n - \phi\|_{\HH}$ is of order $n^{-s}$ with $s\geq 1$ and that  the numerical integrator is of order $\alpha$ (with $\alpha=1$ for the Eulder scheme in Theorem \ref{thm:error_discreteTime}). 
 Then, the total error is of the order $g(n) = n (\Delta x)^{-\alpha} + n^{-s}$. Minimizing it by solving $g'(z) = (\Delta x)^{-\alpha} - sz^{-s-1}=0$, we get  the optimal dimension $n \approx s^{-1/(s+1)} (\Delta x)^{\alpha/(s+1)}$, and the corresponding optimal rate of convergence is $(\Delta x)^{-\alpha s/(s+1)}$. 
 
Taking $\HH$ to be the RKHS,  we obtain the convergence rate of $\mE_{M,\infty}$ from \eqref{eq:cosrFn_bilinear}--\eqref{eq:costFn_data}. 
\end{proof}
% For instant, the order is XXXX $2$ in %Section \ref{sec:cubic} Figure \ref{fig:3x}, where we use a trapezoid integrator. 

\begin{remark}[High dimensional case by Monte Carlo] When the space variable $x\in \R^d$ is high-dimensional (with $d\geq 4$), it becomes impractical to have data on mesh-grids since the data size increases exponentially in $d$. It is natural to consider data consisting of samples of particles and approximate the integrals by Monte Carlo. Our algorithm applies directly. We conjecture the optimal convergence rate to be $s/(2s+1)$ and leave it as future work. 
\end{remark}
% \begin{remark} If we use the RKHS basis, $A=Id$. There is no need of coercivity, but the task is to show that this is possible, i.e. show that $\widebar G_T$ is (semi)-positive definite.  When it is strictly positive definite, we get the coercivity condition. In the identifiability paper, we make these points clear. \QL{Actually I used the un-normalized RKHS, hence the coercivity condition is still needed.} \end{remark}

% \begin{theorem}[Noisy Data] Assume that the data are noisy, with white noise (Or Gaussian noise at each grid point). Denote the estimator by $\widetilde \phi_n$ Then we have XXXX\FL{The noise should not be white. The observations should be another sequence of measure. Numerical tests.}
% \begin{align}\label{eq:numErr_noisy} \E[\widetilde \phi_n - \widehat \phi_n] & \\
% E \| \widetilde \phi_{n,M,L} - \widehat \phi_n \|_{\HH} & \leq 2{c_{\mH,T}}^{-1}\left(c^b \sqrt{n} + c^A n\norm{\widehat \phi_n}_{\HH} \right) (\Delta x +  \Delta t ) ,
% \end{align}
% \end{theorem}

\subsection{Numerical error in the normal equation}
The error in the discrete data estimator comes from numerical integrations in space and in time. We prove that the numerical error in the normal matrix and the vector are of the same order as the Riemann sum approximation to the integrals. We only outline the main proof here and leave the technical results in Appendix \ref{appendixA}.
 
Note that for each $t$, the integrals in space are expectations with respect to $u(\cdot,t)$ and that 
\begin{align}\label{eq:intU=1}
|\Omega| \norm{ u}_\WOT \geq   |\Omega| \norm{ u}_\LO\geq  \int_\Omega u(x,t)dx= 1. 
\end{align}

\begin{proposition}\label{Aberror} The numeric error of $A_{n, M, L}$ and $b_{n, M, L}$ in \eqref{AnML} and \eqref{bnML} are bounded by %of order $O( \Delta x) + O(\Delta t)$ as $\Delta x$ and $\Delta t \rightarrow 0$,
\begin{equation}\label{eq:Ab_error}
\begin{aligned}
\| A-A_{n,M,L} \| & \leq  n c^A (\Delta x + \Delta t),\\
\| b-b_{n,M,L} \| & \leq  \sqrt{n} c^b (\Delta x + \Delta t),
\end{aligned} 
\end{equation} 
where the norm for matrix is in the Frobenius sense,  and the constants $c^A := 2|\Omega|(1+|\Omega|)(c^{1,\infty}_\hypspace)^2 \norm{u}_\WOT^2  $ and $c^b:=3   |\Omega|(1+\RadiusOmega+\nu) c^{2,\infty}_\hypspace (\norm{u}_\WOT^2+ \norm{u}_\WBOT)$ are given in Theorem \ref{thm:error_discreteTime}. 
\end{proposition}

\begin{proof}  
 Using the notation $ \overline{\mathcal{D}}(f)$ in \eqref{eq:Euler_order} with $f= (K_{\phi_i}*u) (K_{\phi_j}*u) u$, we have
  \begin{align*} 
      \abs{A_{i,j} - & A^{i,j}_{n, M, L}} \leq  \overline{\mathcal{D}}(f)+ I^A_1, \text{ with }\\ 
    & I^A_1: =   \frac{1}{L} \sum_{l=1}^L \sum_{m=1}^M  
    \abs{     \left[ \left(  (K_{\phi_i}*u) \cdot(K_{\phi_j}*u)     - P_{n, M, L}^i \cdot P_{n, M, L}^j  \right)  u \right](x_m,t_l)      }  \Delta x. 
  \end{align*}
  To apply \eqref{eq:Euler_order} for $ \overline{\mathcal{D}}(f)$, we estimate first $\grad_{x,t}f$ using \eqref{eq:Ki*uKj*u} and \eqref{convnorm}: 
  \begin{align*}
    &\norm{\nabla_{x,t} \left[  (K_{\phi_i}*u) (K_{\phi_j}*u) u(x,t)   \right]}_\LOT \\
\leq  &\, \norm{\nabla_{x,t} \left[ (K_{\phi_i}*u) (K_{\phi_j}*u)  \right]}_\LOT \norm{u}_\LOT 
            + \norm{ \left[   (K_{\phi_i}*u) (K_{\phi_j}*u)  \right]}_\LOT \norm{\nabla_{x,t}u}_\LOT \\
\leq &\, (2|\Omega|\norm{u}_\LOT+ 1) \norm{u}_\WOT  (  c^\infty_\hypspace)^2. 
  \end{align*}
Note that $(2|\Omega|\norm{u}_\LOT+ 1) \leq 3 \Omega|\norm{u}_\LOT$ by \eqref{eq:intU=1}. Thus, we have 
$$\overline{\mathcal{D}}(f) \leq 3 |\Omega|^2\norm{u}_\WOT^2 (c^{1,\infty}_\hypspace)^2 (\Delta x + \Delta t). $$

To estimate $I_1^A$, note that   by  \eqref{eq:P_conv_error} and \eqref{convnorm}, we have 
\[\norm{P_{n, M, L}^i}_\LOT\leq \norm{K_{\phi_i}*u-P_{n, M, L}^i}_\LOT + \norm{K_{\phi_i}*u}_\LOT  \leq   c^{1,\infty}_\hypspace|\Omega| \norm{u}_\WOT \Delta x + c^\infty_\hypspace. 
\]
Thus, $\norm{P_{n, M, L}^i}_\LOT + \norm{K_{\phi_i}*u}_\LOT \leq 2c^{1,\infty}_\hypspace$ when $\Delta x\leq  2\min_i \norm{\phi_i'}_\infty \slash ( c^{1,\infty}_\hypspace|\Omega| \norm{u}_\WOT)$. Then,
  \begin{align*}
       I_1^A \leq  &\, |\Omega |\norm{u}_\infty  \norm{ (K_{\phi_i}*u) \cdot(K_{\phi_j}*u)     - P_{n, M, L}^i \cdot P_{n, M, L}^j   }_\LO\\
  \leq& \, |\Omega |\norm{u}_\infty  \max_i \norm{  K_{\phi_j}*u-P_{n, M, L}^j }_\LOT \max_i \left( \norm{P_{n, M, L}^i}_\LOT + \norm{K_{\phi_i}*u}_\LOT\right)\\ 
 % \leq&\,  |\Omega| \norm{u}_\LOT  c^{1,\infty}_\hypspace  \norm{u}_\WOT  \Delta x   
%  \left[     \norm{K_{\phi_i}*u}_\LOT +      \norm{P_{n, M, L}^j - (K_{\phi_j}*u)}_\LOT +      \norm{K_{\phi_j}*u}_\LOT   \right]\\
 \leq   &\,   |\Omega| \norm{u}_\LOT c^{1,\infty}_\hypspace  \norm{u}_\WOT  \abs{\Omega}\Delta x  2 c^{1,\infty}_\hypspace  \leq   2|\Omega|^2 \norm{u}_\WOT^2 (c^{1,\infty}_\hypspace)^2 \Delta x. 
  \end{align*} 
 Combine the above estimates of $ \overline{\mathcal{D}}(f)$ and $I_1^A$, we have
  \[
  \abs{A_{i,j} -  A^{i,j}_{n, M, L}} \leq  5\abs{\Omega}^2 \norm{u}_\WOT^2 (c^{1,\infty}_\hypspace)^2 (\Delta x + \Delta t). 
  \]
  and the bound for $ \norm{A_{n, M, L} - A}$ in \eqref{eq:Ab_error} follows.

Next we analyze $ \norm{b_{n, M, L} - b}$. Using $ \overline{\mathcal{D}}(f)$ in \eqref{eq:Euler_order} with $f= u ( \grad\cdot K_{\phi_i}) * u$, we have 
\begin{equation}\label{eq:I^b}
  \begin{aligned}
    \abs{b_i - & b^i_{n, M, L}}\leq     \overline{\mathcal{D}}( u ( \grad\cdot K_{\phi_i} * u) )+ I_1^b + I_2^b,  \text{ with }\\
      I_1^b & :=  \abs{  \frac{1}{T}       \int_\Omega \int_0^T \partial_tu \Phi_i * u \ dxdt -    \sum_{l=1}^L \sum_{m=1}^M \left[ \widehat{\partial_t u} \Phi_i * u\right] (x_m, t_l)\Delta x\Delta t } , \\ %    \overline{\mathcal{D}}( \partial_tu \Phi_i * u )  
   I_2^b & :=  \frac{1}{L} \sum_{l,m=1}^{L,M} %\sum_{m=1}^M  
    \abs{     \left[ \widehat{\partial_t u} \Phi_i * u + \nu  u \left(\grad\cdot K_{\phi_i}*u\right)     - 
                         \widehat{\partial_t u} \,Q_{n, M, L}^i +      \nu u R_{n, M, L}^i      \right](x_m,t_l)
            }    \Delta x. 
  \end{aligned}
  \end{equation}
   By  \eqref{eq:Euler_order} and the gradient estimate \eqref{eq:nablaK*u}, we have % (\FL{We need $\phi_i\in W^{2,\infty} $ here }. )
  \begin{align*}
    \overline{\mathcal{D}}( u ( \grad\cdot K_{\phi_i} * u) )  \leq    |\Omega|  \norm{ u}_\WOT c^{1,\infty}_\hypspace (1+\norm{u}_\LOT) (\Delta x + \Delta t). 
        % &\frac{|\Omega|}{\sqrt{3}}\sqrt{\Delta x^2 + \Delta t^2} \norm{\nabla_{x,t} \left[ \partial_tu\left(\Phi_i * u\right) + \grad u\cdot \left(K_{\phi_i}*u\right) \right]}_\WOT \\
  \end{align*}  
To estimate $I_2^b$, note that $\widehat{\partial_t u}$ in \eqref{eq:PQ} satisfies $\norm{\widehat{\partial_t u}}_\infty \leq \norm{u}_\WOT$. Then,  by \eqref{eq:P_conv_error},  we have 
  \begin{align*}
  \norm{ \widehat{\partial_t u} \left(\Phi_i * u\right) -  \widehat{\partial_t u} \,Q_{n, M, L}^i}_\LOT
     &\leq    \norm{u}_{\WOT}\abs{\Phi_i * u -  \,Q_{n, M, L}^i} \leq  |\Omega| (1+\RadiusOmega) c^\infty_\hypspace \norm{u}_\WOT^2 \Delta x,      \\
 \nu \norm{ u \left(\grad\cdot K_{\phi_i}*u\right) -   u R_{n, M, L}^i}_\LOT 
     &\leq  \nu  \norm{u}_{\LOT}\abs{ \grad\cdot K_{\phi_i}*u - R_{n, M, L}^i}  \leq  \nu \abs{\Omega}c^{2,\infty}_\hypspace \norm{u}_\WOT^2\Delta x.
  \end{align*}
  Hence 
  \[ I_2^b \leq   |\Omega| (1+\RadiusOmega+\nu) c^{2,\infty}_\hypspace \norm{u}_\WOT^2 \Delta x . \]
Together with the estimates of $I_1^b$ in Lemma \ref{lemma:b_time}, we have  
  \begin{align*}
 \abs{b_i -  b^i_{n, M, L}}\leq &3  c^{2,\infty}_\hypspace |\Omega|(1+\RadiusOmega+ \nu) (\norm{u}_\WOT^2+ \norm{u}_\WBOT) (\Delta x + \Delta t),
  \end{align*} 
The estimate for $   \norm{b - b_{n, M, L}}$ in \eqref{eq:Ab_error} follows. 
\end{proof}

%%%%%%%%%%%%%%==== section ======  %%%%%%%%%%
%%%%%%%%%%%%%%==== section ======  %%%%%%%%%%
%%%%%%%%%%%%%%==== section ======  %%%%%%%%%%
\section{Numerical examples}\label{sec:num}
We demonstrate the effectiveness of our algorithm using synthetic data for  examples with three typical types of interaction kernels: the granular media model with a smooth kernel (Section \ref{sec:cubic}),  the opinion dynamics with a piecewise linear kernel (Section \ref{sec:OD}),  and the aggregation-diffusion with a singular repulsive-attractive kernel (Section \ref{sec:RA}). 
In each of these examples, our algorithm leads to accurate estimators that can reproduce solutions and free energy almost perfectly.  Our estimator achieves the optimal rate of convergence for the smooth cubic potential, and sub-optimal rates for the opinion dynamics and the singular repulsive-attractive potential.   
%\QL{Cubic potential? shall we change the name in Section 4.2}?
We first specify the numerical settings: the numerical scheme for data generation, the choice of parameters in the learning algorithm, and the assessment of the estimators. 
%. First, we need to generate the solution $u$ from given potential $\Phi$. We will also introduce the analysis pattern for the later examples. The rest of this section will be devoted to the result analysis for different potentials, including the opinion dynamic potential, cubic potential and the repulsive-attractive potential. 

%%%%%%%%%%%%%%==== section ======  %%%%%%%%%%
\subsection{Numerical settings}

\paragraph{Settings for data generation} \label{sec:SPCC}
We solve the mean-field equation by the semi-implicit Structure Preserving scheme of Chang-Cooper (SPCC) scheme for nonlinear Fokker-Planck equations introduced in Pareschi and Zenella \cite{pareschi2018_StructurePreserving}. The SPCC preserves the steady state and the density properties of the solution $u$ such as non-negativity. In particular, for the explicit SPCC scheme, we need $d t \leq \frac{d x^2}{2(Cd x + \nu)}$, where $C = \sup_{\Omega\times [0, T]} |\grad \Phi * u|$; for semi-implicit scheme, we need $d t \leq \frac{d x}{2C}$.  

In all the three examples, the data are sparse observations from the ``true" solution on $\Omega \times [0,T]$ with a fine space-time mesh with $d t = 0.001$ and $d x = \frac{b-a}{3000}$, where $T = 1$ and $\Omega = [a,b]$ with $a=-10$ and $b=10$. This setting preserves the steady state and the non-negativity of the solution $u$. The data are observed at every 10 space time mesh, i.e. $\Delta x=10 dx$ (or equivalently, $M=300$). To compute the rate of convergence, we down-sample the data further t†o have a sequence of $\Delta x= kdx$ with $ k\in \{10,12,15,20,24,30,50,60,75,100 \} $, correspondingly, we have $ M\in \{300, 250,200,150,120,100,60,50,40,30\} $).  

%\QL{I changed the code to apply this setting}

%\QL{In all the three examples, the data are the solutions on $\Omega \times [0,T]$ with a fine space-time mesh with $d t = 0.001$ and $d x = \frac{b-a}{200}$, where $T = 1$ and $\Omega = [a,b]$ with $a=-10$ and $b=10$. This setting preserves the steady state and the non-negativity of the solution $u$. To compute the rate of convergence, we down-sample a much finer data with same time mesh setting and $dx = \frac{b-a}{3000}$ to have a sequence of $\Delta x= kdx$ with $ k\in \{10,12,15,20,24,30,50,60,75,100 \} $, correspondingly, we have $ M\in \{300, 250,200,150,120,100,60,50,40,30\} $).  }

We summarize these settings in Table \ref{tab:num_settings}. 
\begin{table}[htb] 
{\small
	\begin{center} 
		\caption{ \, Numerical settings in data generation and inference for all examples.} \label{tab:num_settings} \vspace{-1mm}
		\begin{tabular}{ l  l }
		\toprule % \hline
			Notation   &  Description \\  \hline
	                $[0,T]=[0,1] $ and $\Omega = [-10, 10]$    & time interval and space domain \\
	                $d t = 0.001$ and $d x = 20/3000$     & time step and space mesh size of true solution\\
	                \multirow{2}{*}{$\Delta t = dt$ and $\Delta x = k dx$ }   & time step and space mesh size of data\\	
		       	                                                                          &  with $ k\in \{10,12,15,20,24,30,50,60,75,100\} $   \\
	                $r_0=0, r_m=10$    &    knots for B-spline    \\
	                $W(u,\widehat u)$  & the Wasserstein distance, defined in \eqref{eq:Wd}\\
	                $E[u, \phi](t) $         & free energy flow, defined in \eqref{eq:FreeEnergy}\\
			\bottomrule	
		\end{tabular}  
	\end{center}
	}
\end{table}

	\noindent\textbf{Settings for inference algorithm} We estimate the interaction kennel $\phi$ by Algorithm \ref{alg:main}. We test both spline basis and RKHS basis. For the spline basis in \eqref{eq:B_spline}, we use $r_0 = 0$ and $r_m = 10$. The range of knot number is $m\in[3, 40]$ for all three examples. We set the degree of the spline to be in $\{0,1,2,3\}$ according to the smoothness of each example. We obtain the RKHS basis functions by solving the eigenvalue problem \eqref{eq:RKHS_basis} as described in Section \ref{sec:basisFun}, and we set the dimension range to be $[2,50]$ for all examples. 
	
%	\FL{The range of knot number is $m\in[m_1, m_2]$ and will be specified for each example. }

We select the optimal dimension of the hypothesis space under adaptive regularization. That is, we find first the optimal regularization constant for each given dimension of hypothesis space as in Section \ref{sec:opt_reg}, then we select the optimal dimension as in Section \ref{sec:opt_dim}. Figure \ref{fig:L&D} demonstrates this process for the opinion dynamics (the plots are similar for other examples).

%\paragraph{Regularization and dimension selection}
%We need to pick the optimal regularization constant $\lambda$ for each dimension $n$, and then determine the optimal $n$ to minimized the regularized error functional $\mE_\lambda$. Examples are given in the Figure \ref{fig:L&D}.
\begin{figure}[htb] 
  \centering 
  \subfigure[Use the L-curve to find optimal regularization constant]{
    \includegraphics[width=0.51\textwidth]{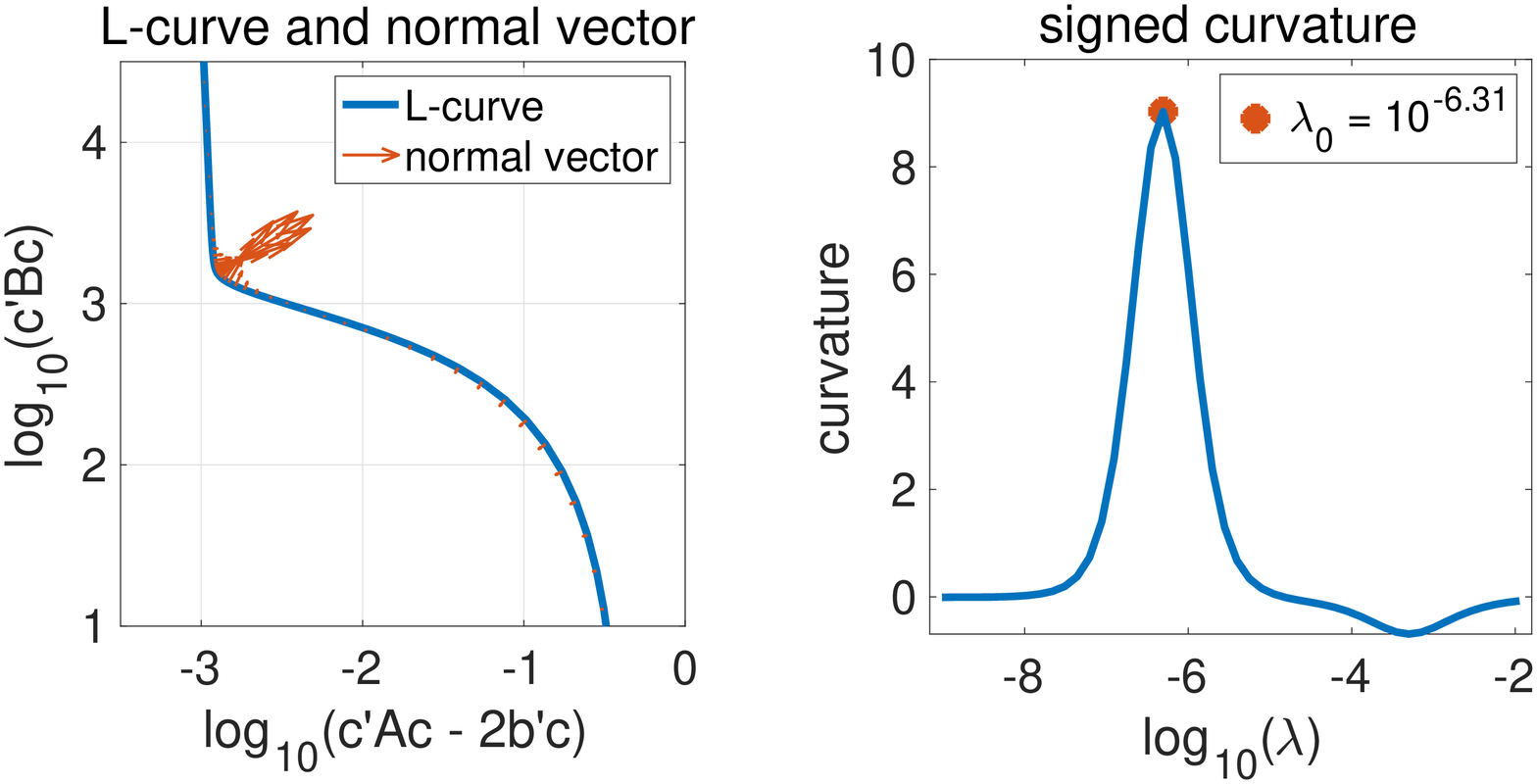}
  }
  \ \ 
  \subfigure[Find the optimal dimension]{
    \includegraphics[width=0.42\textwidth]{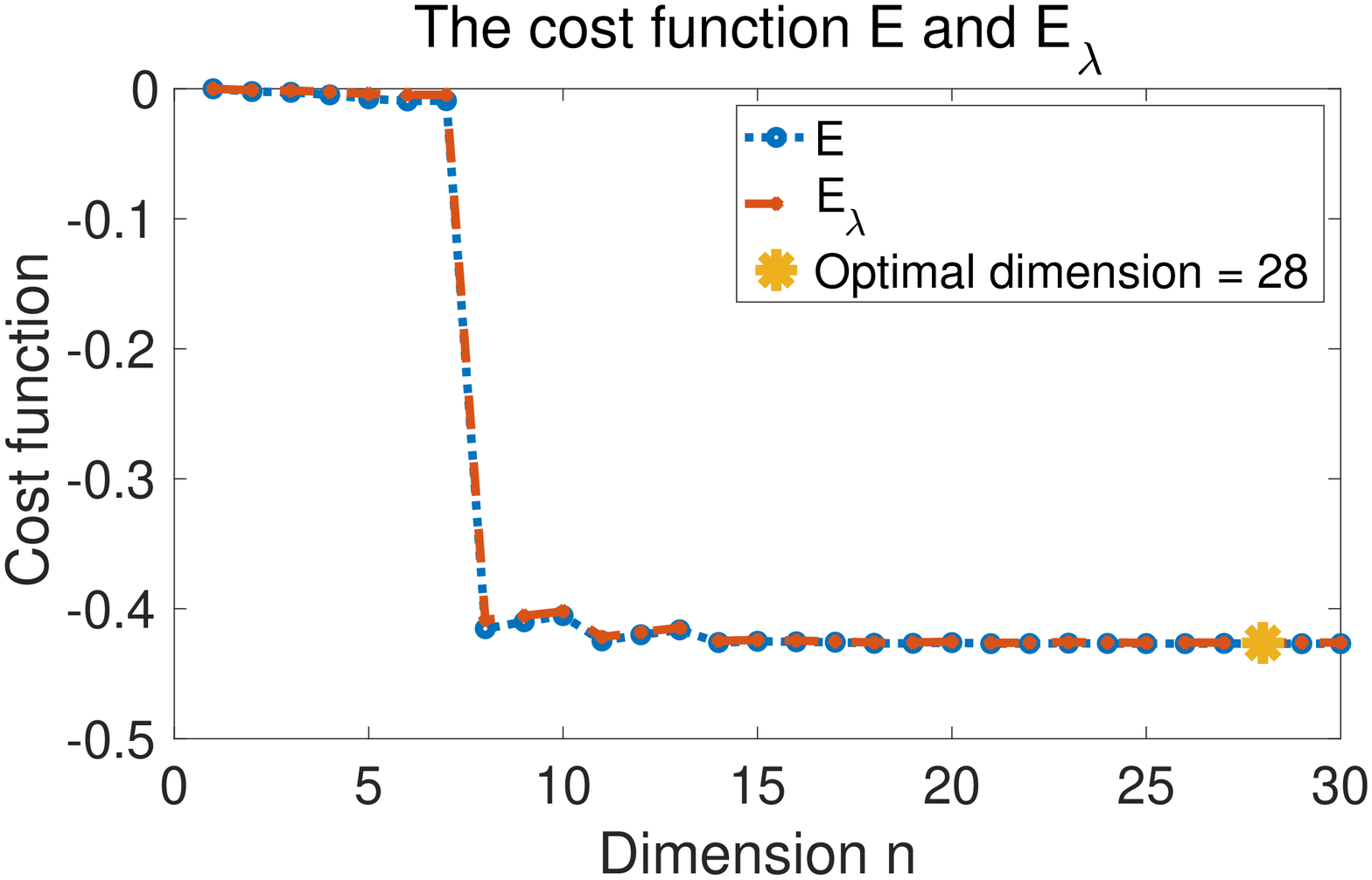}
  }
  \vspace{-3mm}
  \caption{Regularization and selection of optimal dimension for the opinion dynamics (see Section \ref{sec:OD}) using degree 1 spline basis with knot number 30. (a) the use of L-curve to find the optimal regularization constant and (b) shows the selection of optimal dimension for the hypothesis space.
  }
  \label{fig:L&D}
\end{figure}

\paragraph{Results assessment in a typical estimation} We assess the estimator in a typical estimation with $M=200$ in three plots: comparison of the estimator with the truth, the Wasserstein distance between the estimated and true solutions, and the true and reproduced free energy flows.  

\begin{itemize}[leftmargin=4mm]
\item \textbf{Estimated and true kernels} We compare the true and estimated kernels by plotting them side-by-side, together with the density $\rhoT$. The estimated kernels are from either B-spline basis functions or RKHS basis functions, with optimal dimension provided in the context. We also give the relative RKHS error and relative $L^2(\rhoT)$ error.

\item \textbf{Wasserstein Distance} The solutions from the true and estimated kernels can not be distinguished by eyes. To compare them, we compute the 2-Wasserstein distance between them. We consider two sets of solutions, starting from either the original or a new initial condition $\widetilde u_0$. We set $\tilde u_0$  to be the average of the density functions of $\mN(2, 1)$ and $\mN(-2, 1)$, whose major mass is in the support of $\rhoT$. Recall that the 2-Wasserstein distance $W_2(f, g)$ of two probability density functions $f$ and $g$ over $\Omega$ with second order moments is given by
\begin{equation}\label{eq:Wd}
	W_2(f, g) := \brak{\inf_{\gamma \in \Gamma(f, g)}\int_{\Omega \times \Omega} |x-y|^2d\gamma(x,y)}^{1/2},
\end{equation}
where $\Gamma(f, g)$ denotes the set of all measures on $\Omega \times \Omega$ with $f$ and $g$ as marginals (see e.g.,\cite{Ambrosio2008GradientFlow}). We use the  numerical method for Wasserstein distance as in \cite{Kolbe2020WassersteinDistance}. This method is based on an observation in \cite{carrillo2004wasserstein}. More precisely, suppose $F$ is the probability distribution induced by the density $f$ and define its pseudo inverse by setting, for $\alpha \in (0,1)$, $F^{-1}(\alpha) = \inf\{x:F(x)>\alpha\}.$ Similarly we have $G$ and $G^{-1}$. Then the $L^2$ distance of the pseudo inverse functions
\begin{equation*}
d_2(f,g) = \brak{\int_0^1 [F^{-1}(\alpha) - G^{-1}(\alpha)]^2 d\alpha}^{1/2}
\end{equation*}
is equal to the 2-Wasserstein distance $W_2(f,g)$.

\item \textbf{Free Energy}  We also compare the true and estimated free energy flows. The free energy, whose Wasserstein gradient is the right hand side of the mean-field equation \cite{carrillo2019_AggregationdiffusionEquations}, is defined by
\begin{equation}\label{eq:FreeEnergy}
%	E[u, \phi](t) = \underbrace{ \nu \int_{\R^d} u\ \text{log}(u) dx}_{\text{entropy } \mathcal{U}(u)} + \underbrace{ \int_{\R^d} u (u*\Phi)dx}_{\text{ Interaction energy } \mathcal{W}(u)}. 
	E[u, \phi](t) = \nu \int_{\R^d} u\ \text{log}(u) dx + \int_{\R^d} u (u*\Phi)dx,\text{ with } \Phi(r)=\int_0^r \phi(s)ds.
\end{equation}
%The solution $u$ can be viewed as a trajectory of a Wasserstein gradient flow with the given free energy and certain initial value $u_0(x)$. See more in  \cite{carrillo2019_AggregationdiffusionEquations}. 
The true and estimated free energy flows are $E[u, \phi](t)$ and $E[\widehat u, \widehat \phi](t)$, respectively. 
\end{itemize}

\paragraph{Rate of convergence} We test the rate of convergence of the estimator in $L^2(\rhoT)$ error and empirical error functional $\mE_{M,L}$ in \eqref{eq:costFn_data}, as $\Delta x$ changes. We consider the downsampled data with $\Delta x =k dx$, where $ k\in \{10,12,15,20,24,30,50,60,75,100 \}$. %or equivalently, $M\in \{300, 250,200,150,120,100,60,50,40,30\} $. 
 We use the spline basis, because the data-adaptive RKHS basis is not suitable for such a test. 
 % To control the effect of dimension, we pick the optimal dimension based on the finest sample with $k = 10$, and fix this $n_{opt}$ for other samples. \FL{We should select optimal dimension according to each dataset.}
 
% \FL{We find the optimal knot number $n_{opt}$ (with  optimal regularization), for each data set. }

For each estimator, we compute the $\|\phi - \hat \phi\|_{L^2(\rhoT)} $ by Riemann sum approximation.  The measure $\rhoT$, defined in \eqref{eq:rhoavg}, is approximated from data by $\rho_L^M$ in \eqref{rhoAppr}, using the data with the finest mesh $\Delta x = 10dx$ (equivalently, M = 300). We compute the RKHS error as in \eqref{errfcnl_Ab} with a little modification: we use the $A$ computed from the finest data and individual $b$ and $c$ for different samples.

The rate of convergence of the error functional is twice the rate of the estimator in the RKHS metric, because of \eqref{eq:cosrFn_bilinear}. From data, we obtain a sequence of values for the error functional $\mE_{M,L}(\hat \phi_{n,M,L})$ as in \eqref{eq:costFn_data}. We compute its convergence rate $\beta$ by optimization: suppose the error functional has the form $E_k=a\Delta x_k^\beta-\gamma$ with the multiplicative constant $a$ and $\gamma =\rkhs{\phi,\phi}_{\GbarT}$ unknown, we compute $\beta$ by 
$\beta, \gamma, a)=\argmin{\beta, c, a} \sum_k |\log (E_k+\gamma)- \beta \log \Delta x_k - \log a|^2$.

\begin{figure}[htb]
  \centering 
  \subfigure[The solution $u(x,t)$]{    \includegraphics[width=0.32\textwidth]{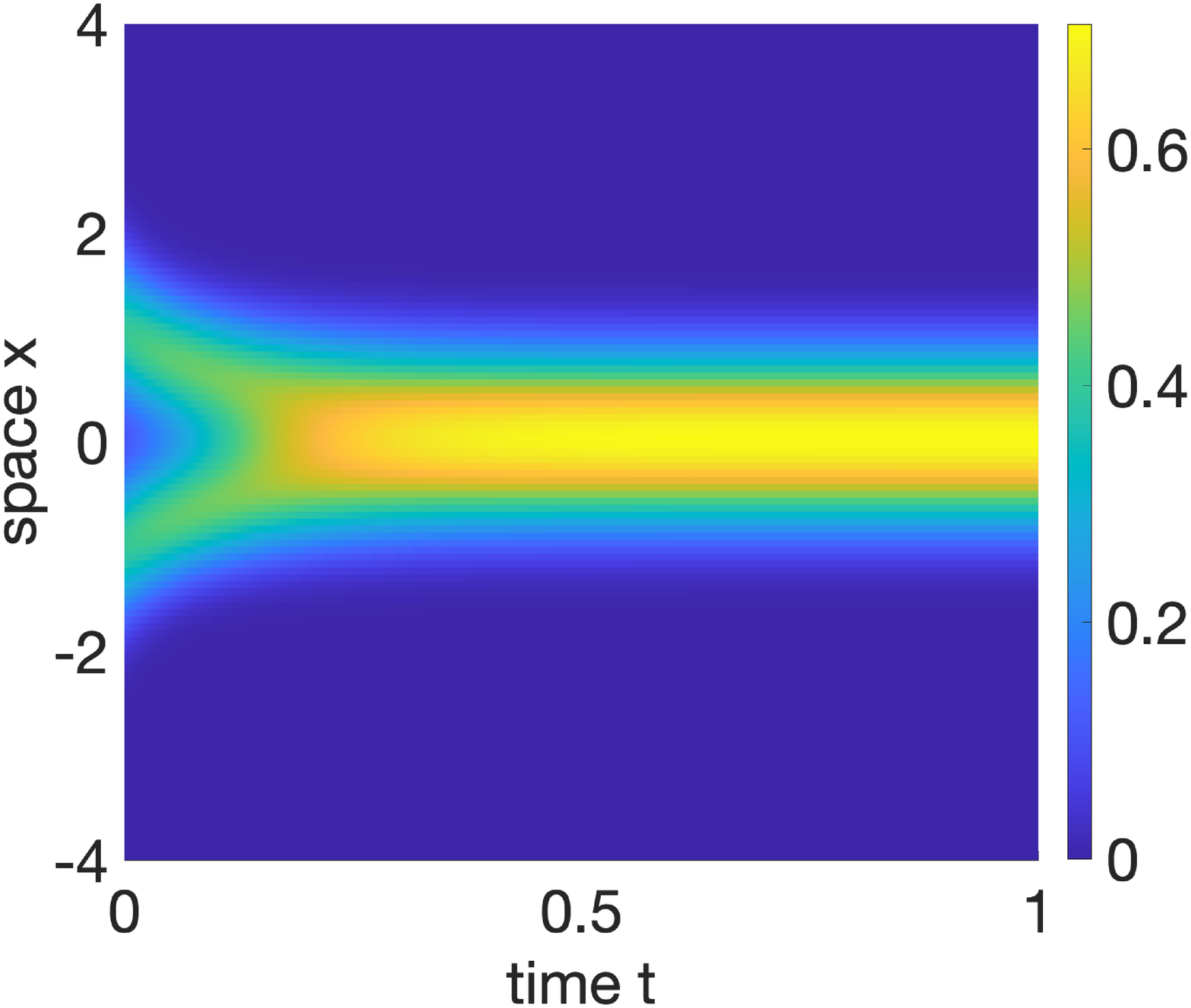}  }
  \subfigure[Wasserstein distance $W(u, \widehat u)( t)$]{     \includegraphics[width=0.3\textwidth]{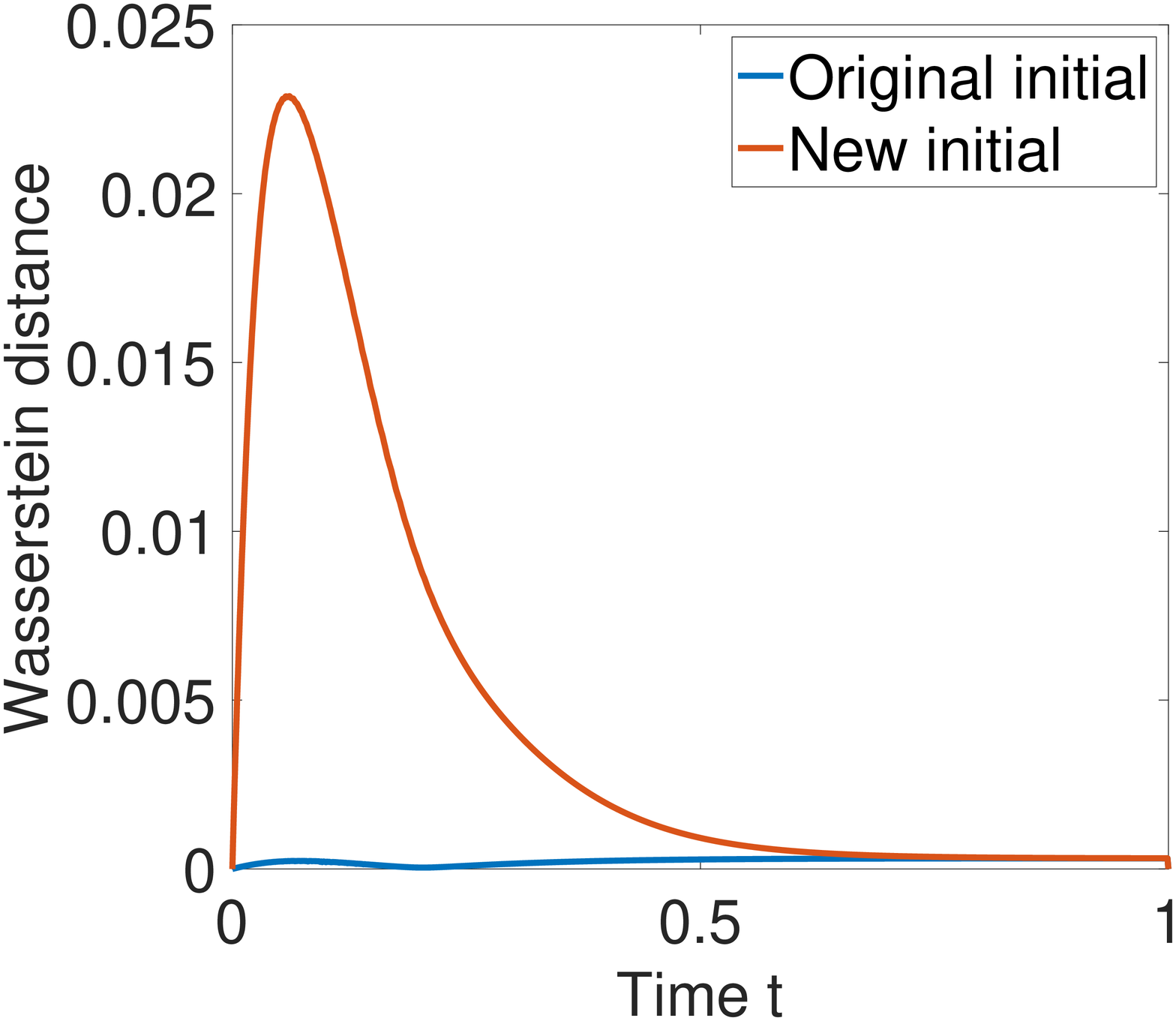}  }\ \ \ 
    \subfigure[Convergence rate of $L^2(\rhoT)$ error]{    \includegraphics[width=0.3\textwidth]{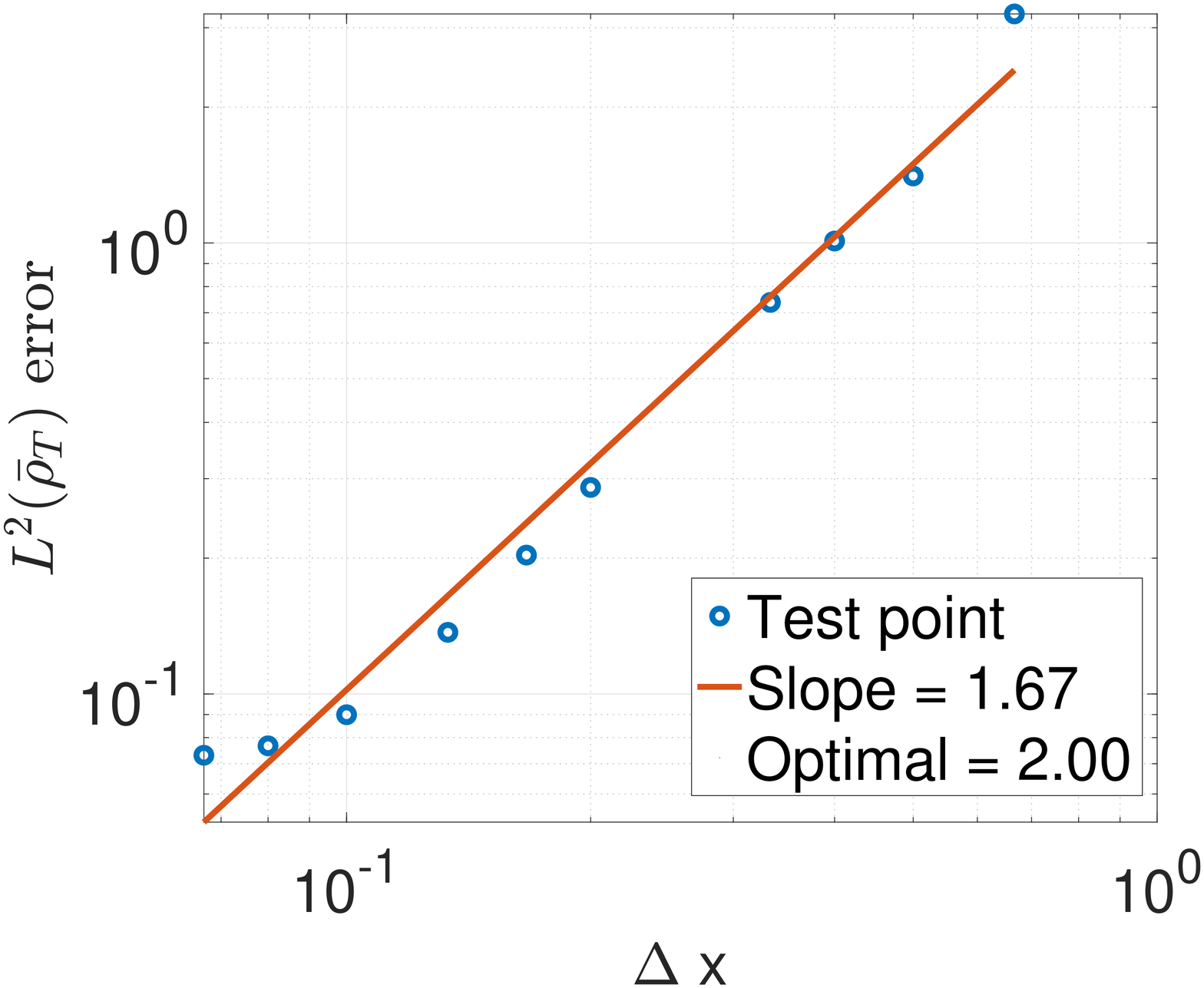}  }
    \subfigure[Estimators of $\phi$]{     \includegraphics[width=0.32\textwidth]{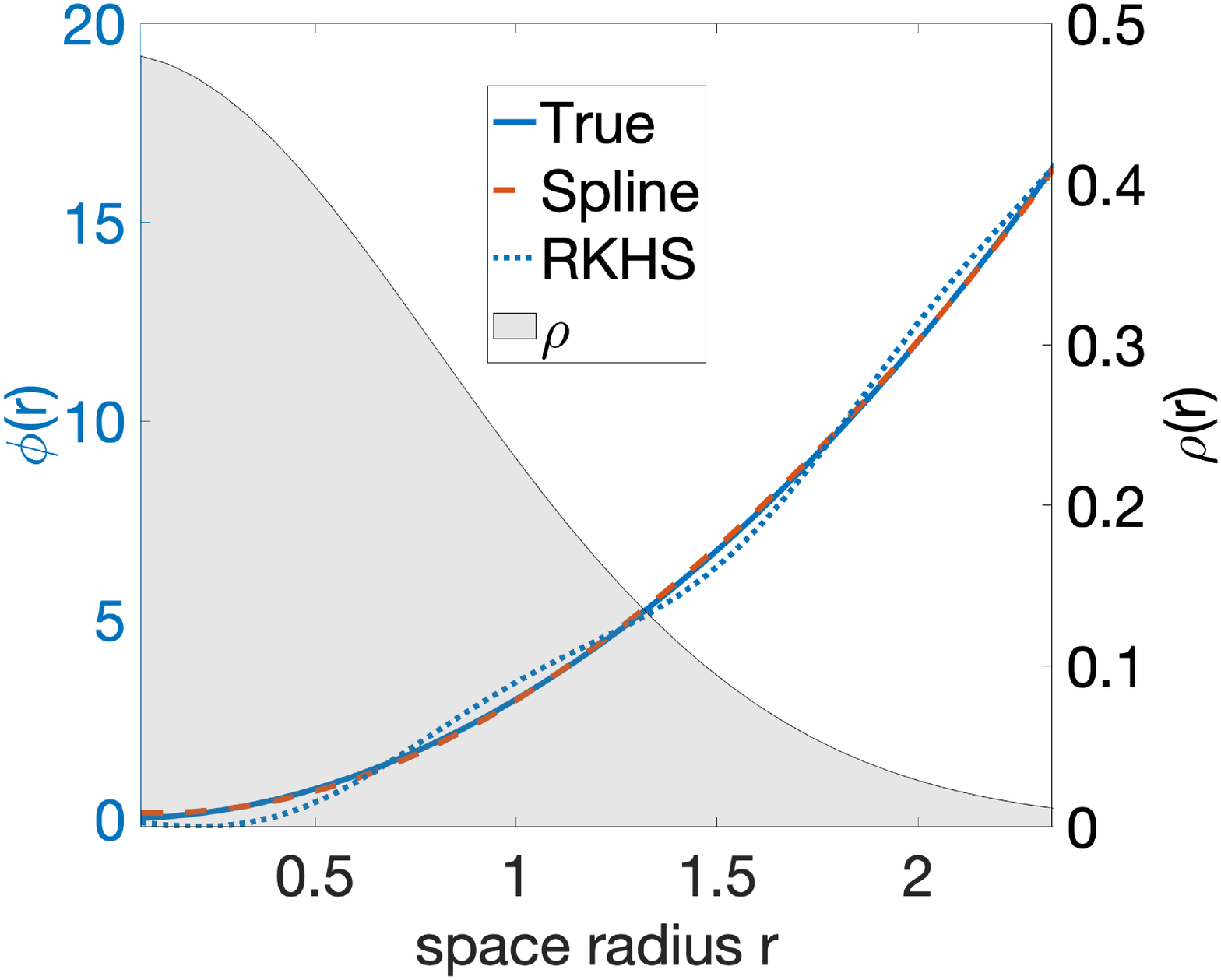}   }
    \subfigure[Free energy]{    \includegraphics[width=0.3\textwidth]{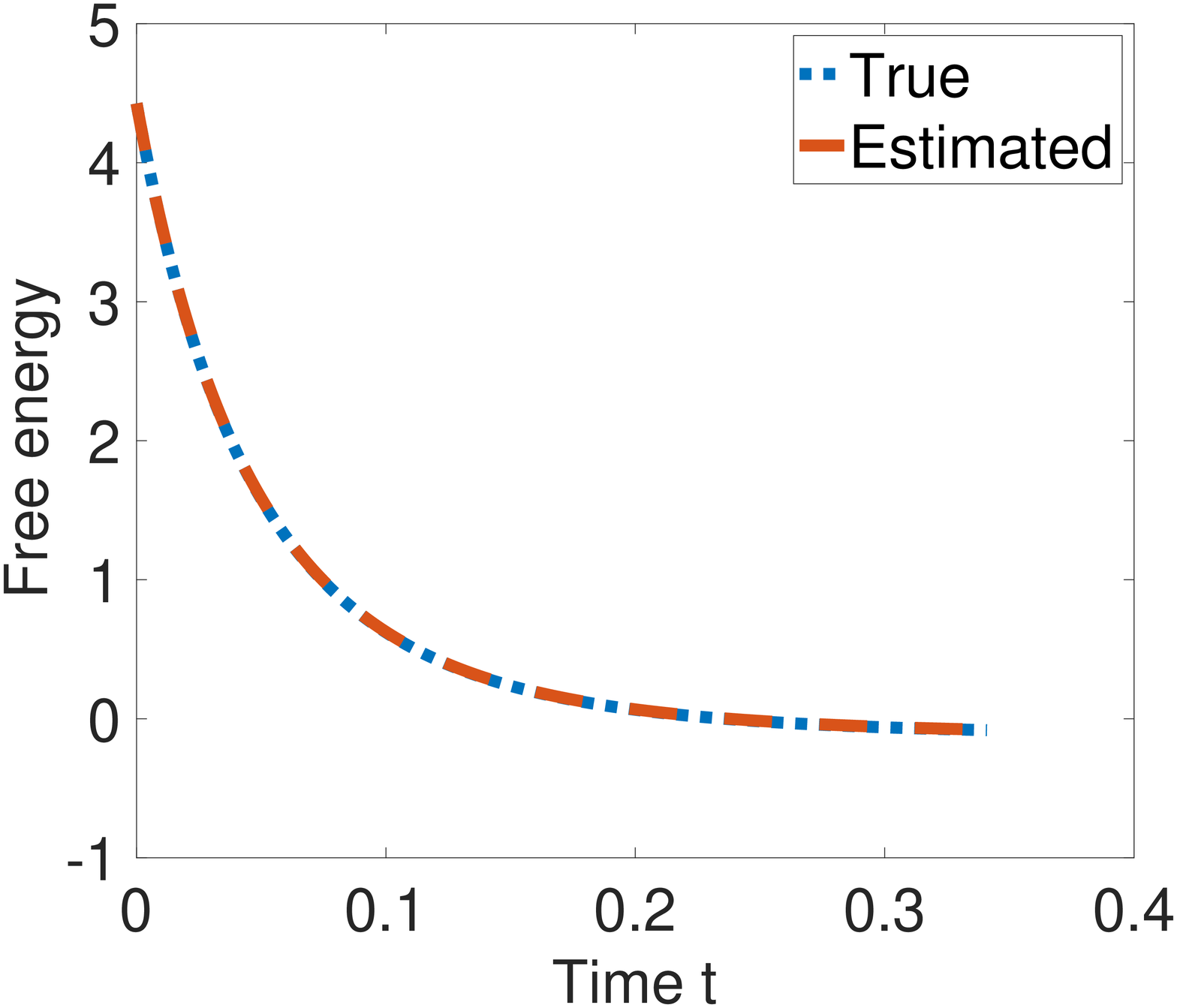}   }\ \ \ 
    \subfigure[Convergence rate of $\mE_{M,L}$ in $\Delta x$ ]{    \includegraphics[width=0.3\textwidth]{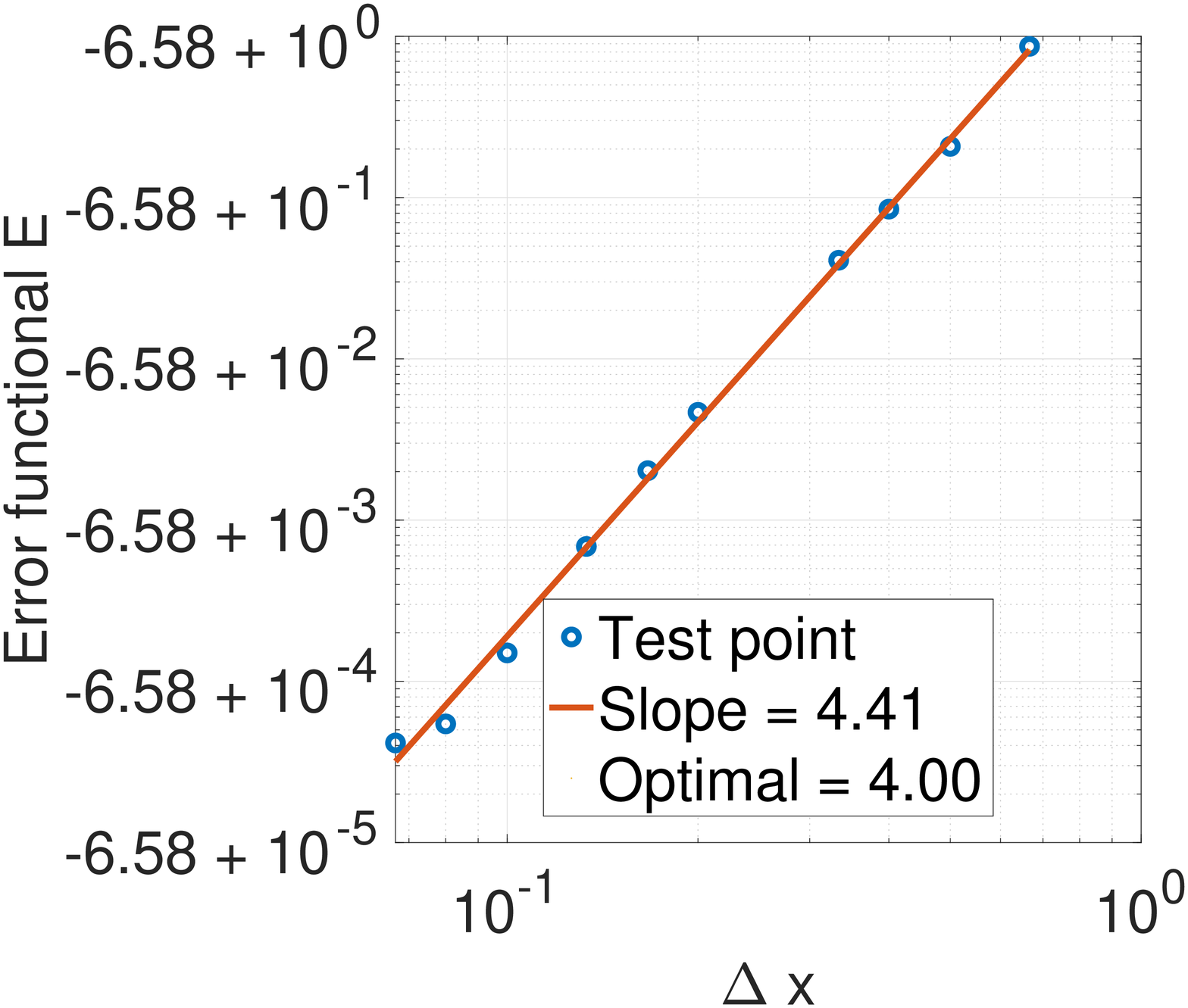}
  }  \vspace{-3mm}
  \caption{Learning results of cubic potential. Subfigure (a) show the solution formulating a steady state. (d) shows the estimated  kernels, superimposed with the empirical density $\rhoT$. The relative errors are shown in Table \ref{tab:3x}. (b) shows that the Wasserstein distance between the solutions are small. In particular, they all tend to the same steady state at large time. In (e) the free energy is well-learned. 
 Subfigure (c) and (f) show the rates of convergence of the $L^2(\rhoT)$ error and the error functional. The two rates are close to the optimal rates in Theorem \ref{thm:rate}. }  \label{fig:3x}
\end{figure}
\vspace{-3mm}
\subsection{Cubic potential }\label{sec:cubic}
The cubic potential $\Phi(x) = |x|^3$ (equivalently, $\phi(r)= 3r^2$) is of special interest for modeling of granular media \cite{malrieu2003_ConvergenceEquilibriumc,cattiaux2007_ProbabilisticApproacha}. Since $\Phi$ is only non-uniformly convex on a single point, the equation \eqref{eq:MFE} possess a unique steady state \cite{cattiaux2007_ProbabilisticApproacha} and thus the SDE \eqref{eq:nSDE} is ergodic. We set $\nu = 1$  and take  $u_0(x)$ to be the average of the densities of $\mN(1, 0.25)$ and $\mN(-1, 0.25)$. 

We use B-spline basis with degree 2, matching the degree of the true kernel. %\FL{Let us state it as a new paragraph, to heighlight the setting.}

Figure \ref{fig:3x} presents the estimation results. Sub-figure (a) shows the solution $u(x,t)$, which is dominated by the diffusion. Subfigure (d) shows that the estimated kernels, either by B-spline or RKHS basis, are close to the true kernel, with relative errors shown in Table \ref{tab:3x}. 
%For spline basis, the relative error in RKHS norm is 0.43\%, and is 1.90\% in $L^2(\rhoT)$ norm. The optimal knot number is 10. For RKHS basis, the relative error in RKHS norm and $L^2(\rhoT)$ are 0.51\% and 7.98\% respectively, and the optimal basis number is 13.  
\vspace{-3mm}
\begin{table}[H]
 \centering
 {\small
		\caption{ \, Relative errors of estimators in Figure \ref{fig:3x}(d).%, with  $\|\phi\|_{L^2(\rhoT)}=XXX$ and $\|\phi\|_{RKHS}=XXX$.
		 } \label{tab:3x} \vspace{-3mm}
		\begin{tabular}{ c | c | c c  }
		\toprule % \hline
			           &                              &  \multicolumn{2}{c}{Relative error}   \\  
		Basis type  & optimal dimension & in $L^2(\rhoT)$ ($\|\phi\|=3.84$)  & in RKHS  ($\|\phi\|=2.57$)   	 \\ \hline
	         B-spline     & 10                          & 1.90\%                        &   0.43\%        \\
	         RKHS       & 13                       &      7.98\%                  &    0.51\%  \\
			\bottomrule	
		\end{tabular}  
}
\end{table}
 Subfigure (b) plots the Wasserstein distances between true and reproduced solutions, showing that the solution with the original initial condition is accurately reproduced. For the new initial condition $\widetilde u_0$, the Wasserstein distance is relatively large at first and then decays to the same level as the original initial condition case. This is because: (1) $\widetilde u_0$, the mixture of $\mN(2, 1)$ and $\mN(-2, 1)$, has a large probability mass outside of the ``well-learned region, the large-probability region of $\rhoT$; (2) the system converge to a unique state state for different initial conditions. 
 Subfigure (e)  shows that the free energy flow is almost perfectly reproduced. Subfigure (c) and (f) show that we achieve near optimal rates of convergence of the estimator in $L^2(\rhoT)$ and the error functional. The two rates are close to the theoretical optimal rates $\frac{\alpha s}{s+1} = 2$ and $\frac{2\alpha s}{s+1} =4$ in Theorem \ref{thm:rate}, where $s=2$ is the degree of the B-spline set according the true kernel and $\alpha=2$ is the order of trapezoidal integration. 
% \QL{For spline basis, the relative error in RKHS norm is 1.26\%, and is 5.68\% in $L^2(\rhoT)$ norm. For RKHS basis, the relative error in RKHS norm and $L^2(\rhoT)$ are $1.12\%$ and $16.88\%$ respectively.}
  
%\FL{We use B-spline basis with degree 2, matching the degree of the true kernel. We test knot numbers ranging from 3 to 20 for the B-splines, and the dimension for RKHS is from 1 to 20. }\QL{I set uniform dimension for all examples}

 \begin{figure}[htb]
  \centering
  \subfigure[The solution $u(x,t)$]{ \includegraphics[width=0.32\textwidth]{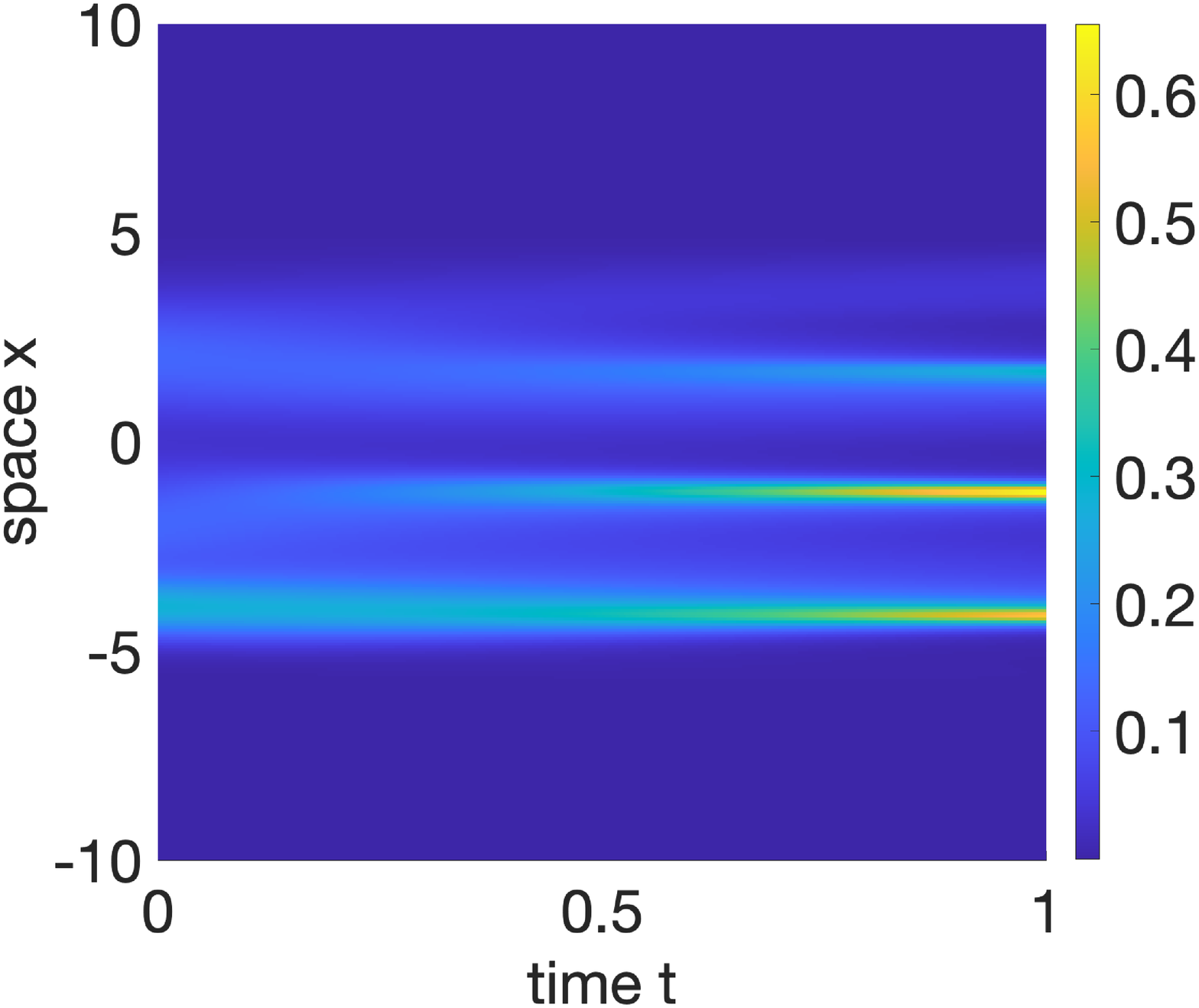}   }
  \subfigure[Wasserstein distance $W(u, \widehat u)( t)$]{ \includegraphics[width=0.3\textwidth]{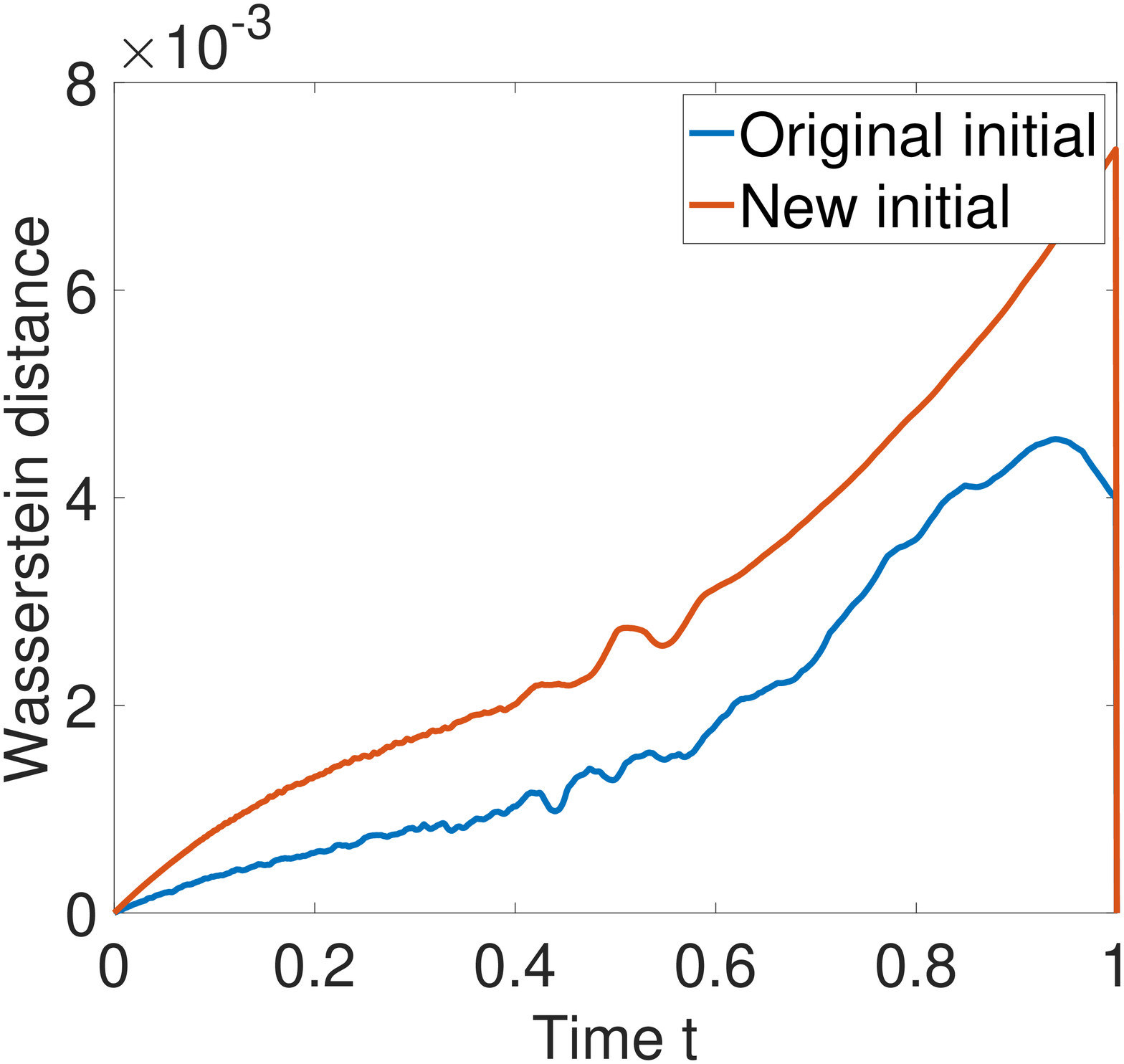} }\ \ \ 
  \subfigure[Convergence rate of $L^2(\rhoT)$ error]{  \includegraphics[width=0.3\textwidth]{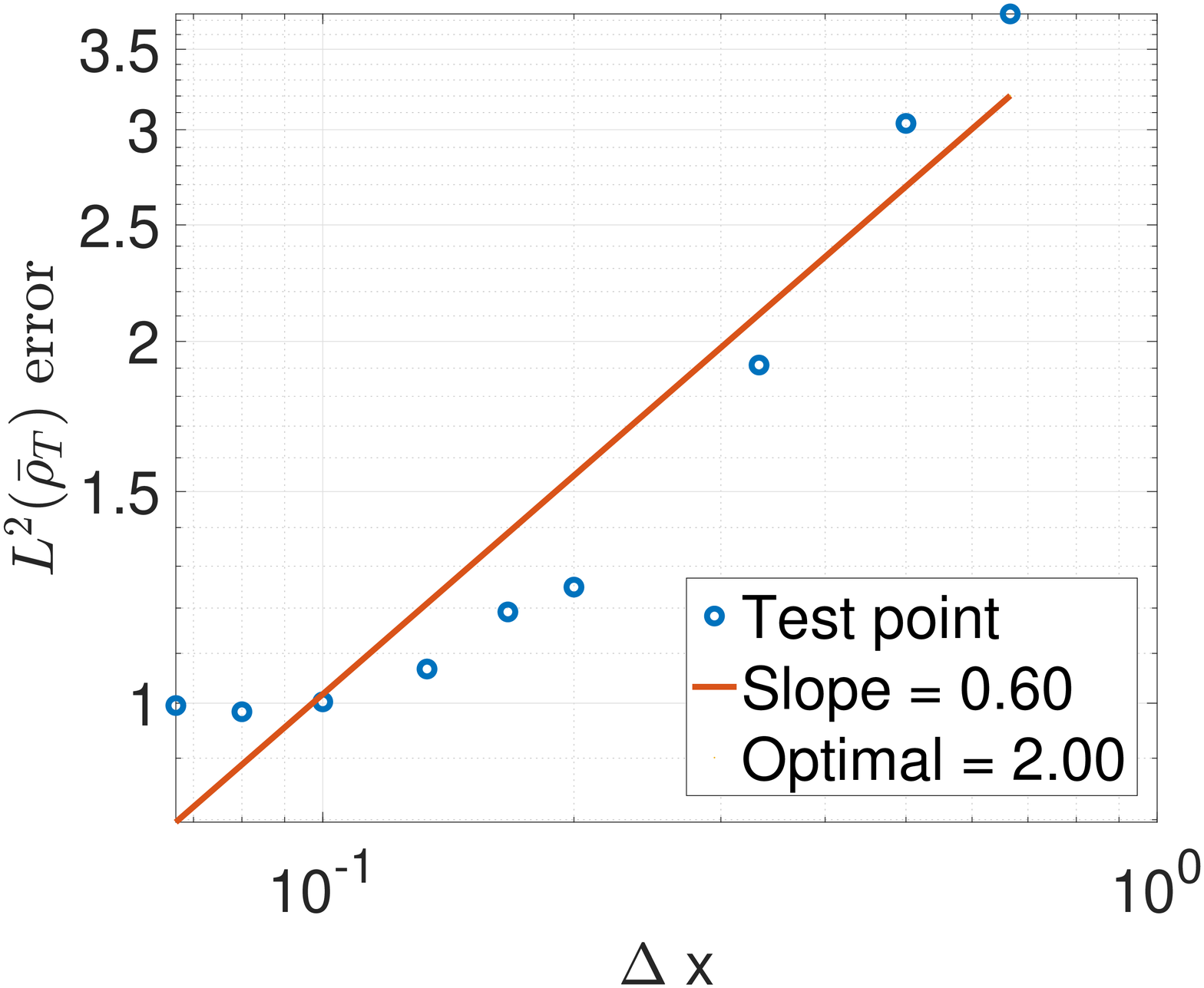}   }
    \subfigure[Estimators of $\phi$]{ \includegraphics[width=0.32\textwidth]{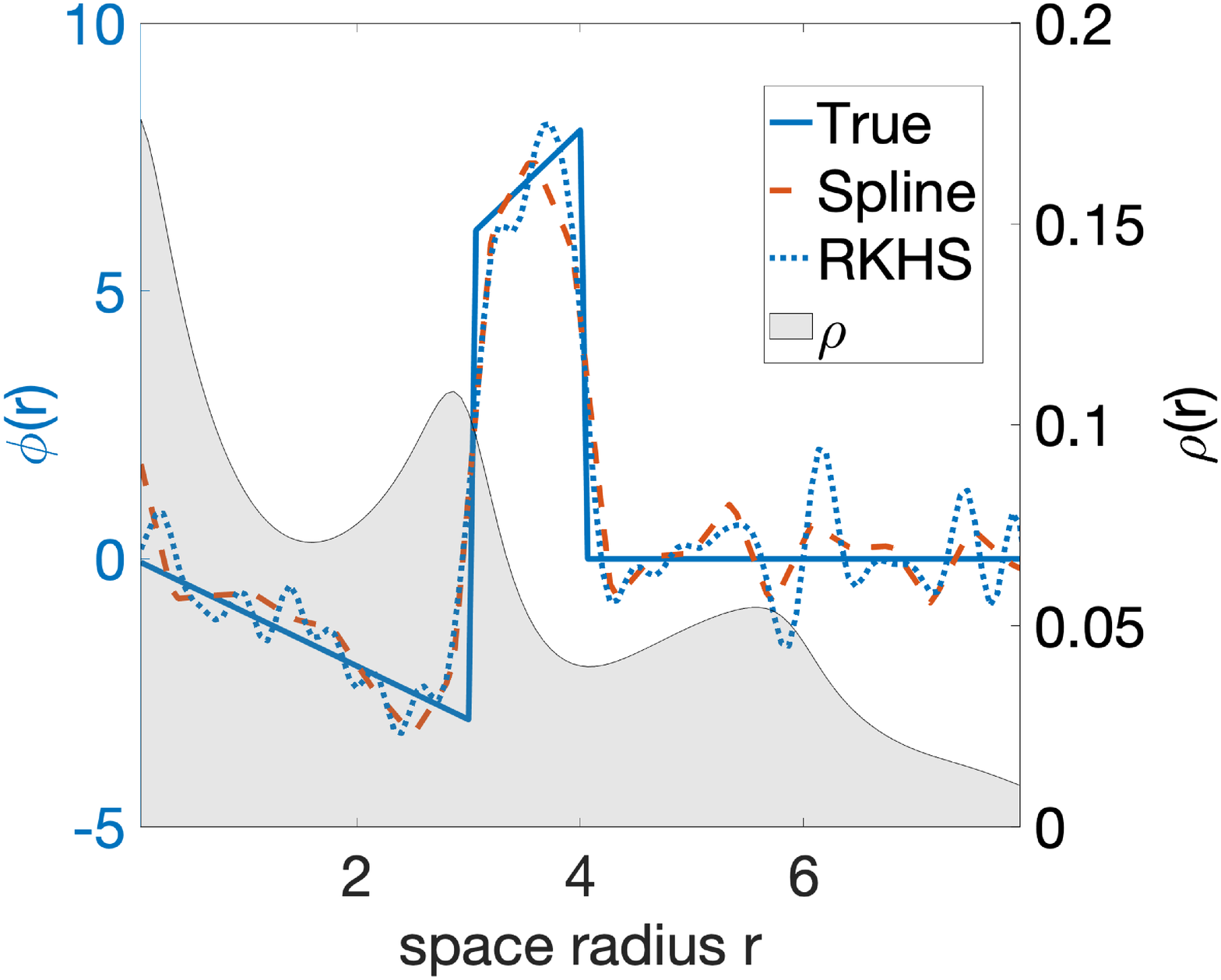}   }
    \subfigure[Free energy ]{ \includegraphics[width=0.3\textwidth]{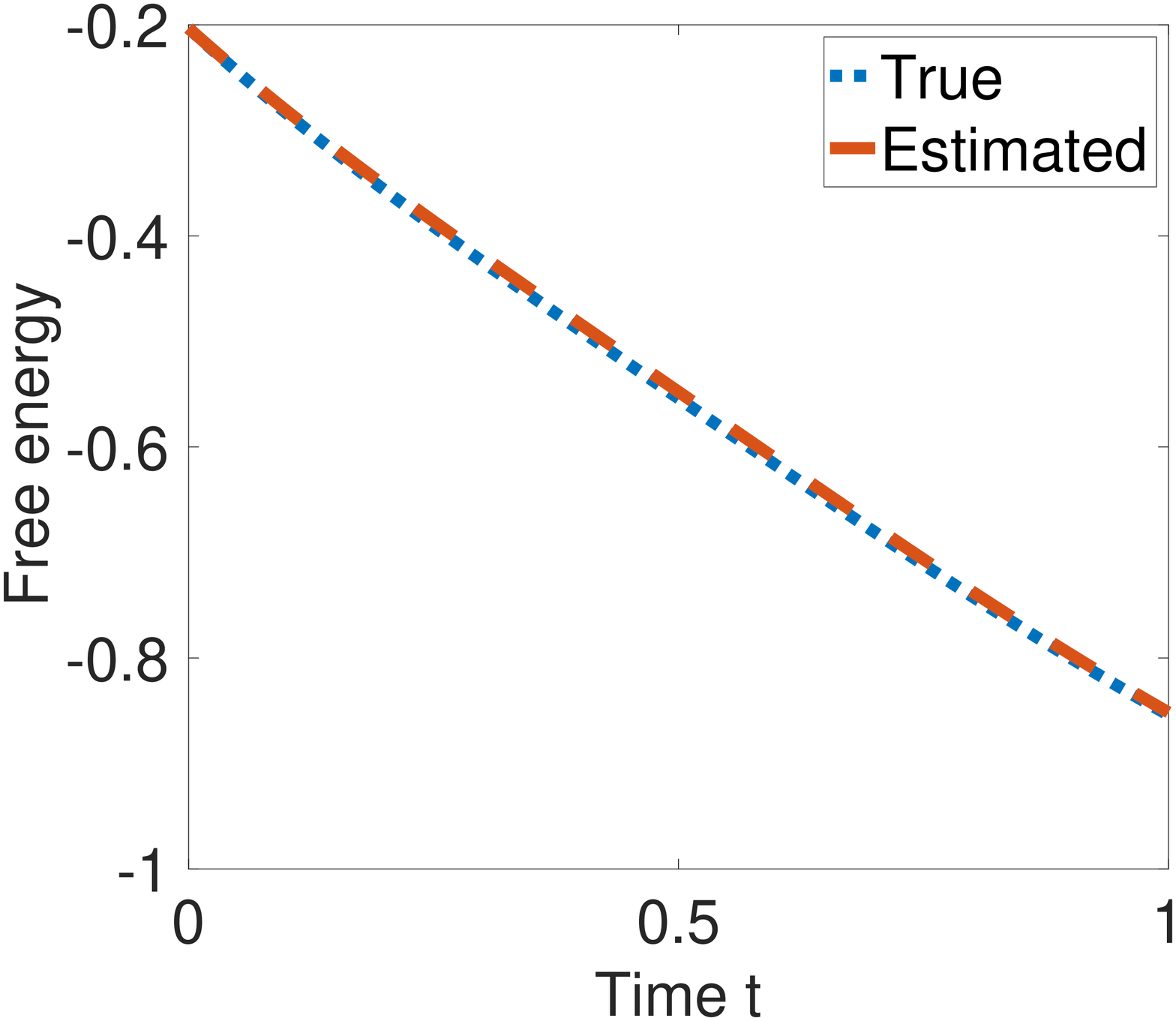}   }\ \ \ 
    \subfigure[Convergence rate of $\mE_{M,L}$ in $\Delta x$]{ \includegraphics[width=0.3\textwidth]{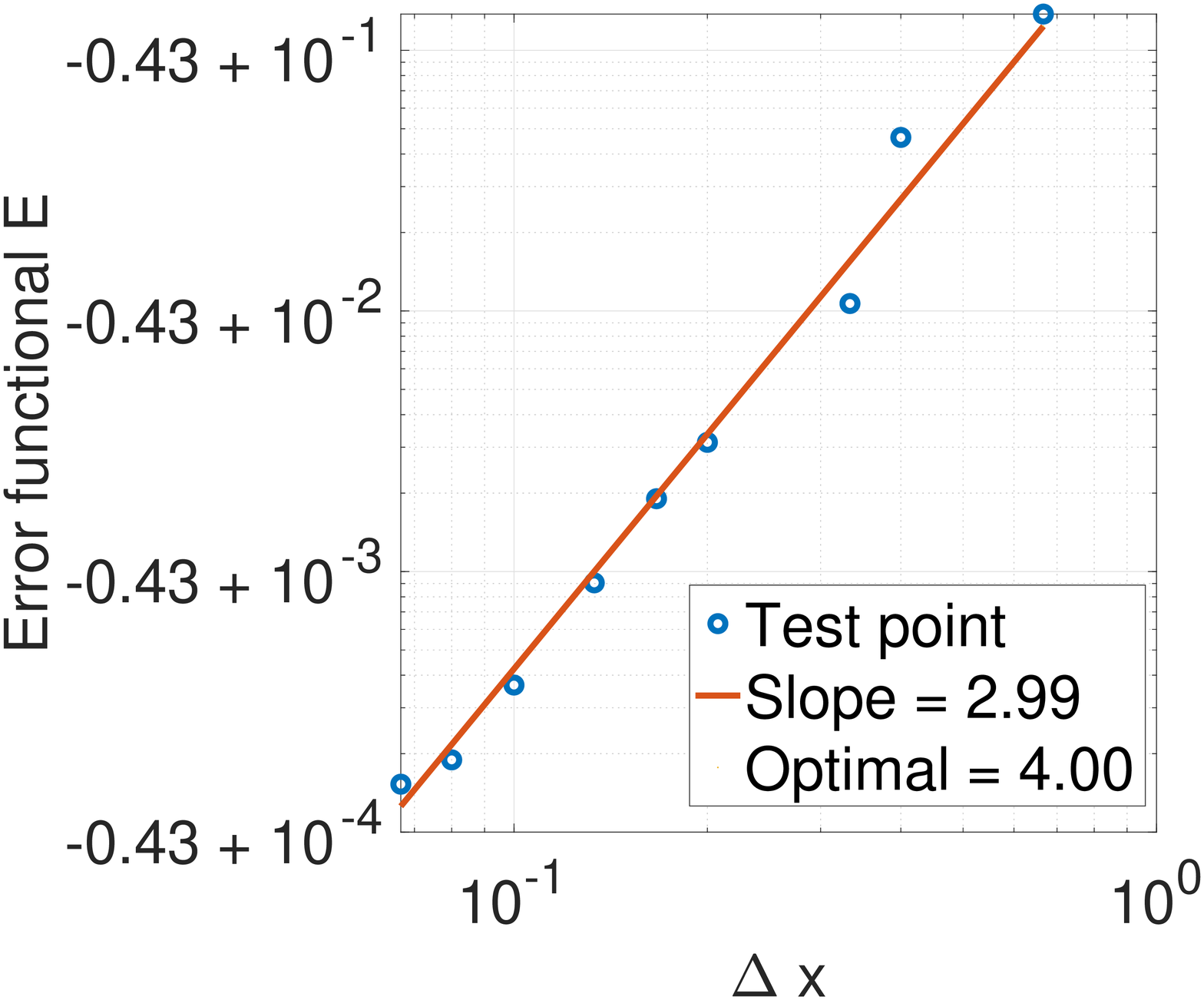}  }
  \vspace{-3mm}
  \caption{Learning results of opinion dynamics.  The solution in (a) formulates clusters. Subfigure (d) shows the estimated kernels, with relative errors shown in Table \ref{tab:OD}. %, by either RKHS or B-spline basis, are close to the true kernel, particularly in the high probability region of $\rhoT$.  
  The Wasserstein distance in (b) shows that the solution is accurately reproduced by the estimated kernel with spline, for both the original and new initial conditions. Subfigure (e) show that the free energy is almost perfectly reproduced. Subfigure (c) and (f) show the rates of convergence of the estimator $L^2(\rhoT)$ and the error functional. Due to the lack of regularity of the true kernel, the optimal rate in Theorem \ref{thm:rate} does not apply.  }  \label{fig:OD}
\end{figure}

%%%%%%%%%%%%%%%%
\subsection{Opinion dynamics}\label{sec:OD}
Opinion dynamics (see \cite{motsch2014_HeterophiliousDynamicsa} and the reference therein) describes the evolution of opinions of agents in social networks. We consider the case when the system formulates clusters of opinions: 
%  for suitable kernels and initial conditions when the viscosity constant $\nu$ is small. 
the interaction function $f(\abs{x}) = \phi(\abs{x})/\abs{x}$ is piecewise constant,   \vspace{-3mm}
\begin{equation*}
f(r) = 
\left\{      \begin{array}{lr}
             -1, & 0\leq r\leq 3, \\
             2,  & 3< r\leq 4,\\
             0,  & 4< r.
             \end{array}
\right.
\end{equation*}
and hence $\phi(\abs{x}) = f(\abs{x})\abs{x}$ is piecewise linear; the initial value $u_0(x)$ is density of the Gaussian mixture $\frac{1}{3} [ \mN(-2, 1)+\mN(-4, 0.5^2) + \mN(2, 1)]$;  the viscosity constant is $\nu = 0.1$. 

We set the degree of spline basis to be 1.

 Figure \ref{fig:OD} presents the estimation results. Subfigure (a) is the solution $u(x,t)$, which shows three clusters forming at time $T=1$. Subfigure (d) shows the estimated and true kernels, with relative errors shown in Table \ref{tab:OD}. 
 % For spline basis, the relative error in RKHS norm is 8.10\%, and is 36.74\% in $L^2(\rhoT)$ norm. The optimal knot number is 28. For RKHS basis, the relative error in RKHS norm and $L^2(\rhoT)$ are 7.46\% and 46.66\% respectively, and the optimal basis number is 40. 
 \vspace{-3mm}
\begin{table}[h]
 \centering
 {\small
		\caption{ \, Relative errors of estimators in Figure \ref{fig:OD}(d).%, with  $\|\phi\|_{L^2(\rhoT)}=XXX$ and $\|\phi\|_{RKHS}=XXX$.
		 } \label{tab:OD} \vspace{-3mm}
		\begin{tabular}{ c | c | c c  }
		\toprule % \hline
			           &                              &  \multicolumn{2}{c}{Relative error}   \\  
		Basis type  & optimal dimension & in $L^2(\rhoT)$ ($\|\phi\|=2.71$)  & in RKHS  ($\|\phi\|=0.65$)   	 \\ \hline
	         B-spline     & 28                          &  36.74\%                     &   8.10\%        \\
	         RKHS       & 40                       &      46.66\%                  &    7.46\%  \\
			\bottomrule	
		\end{tabular}  
}
\end{table}
Subfigure (b) is the Wasserstein distance between true and reproduced solutions, showing that the solutions are accurately reproduced. The Wasserstein distance increases because of the formulating clusters, which lead to singular measures. Subfigure (e)  shows that the free energy flow is almost perfectly reproduced. Subfigure (c) and (f) shows the rates of convergence of the estimator in $L^2(\rhoT)$ and the error functional. Due to the lack of regularity of $\phi$, % (recall that we need it to be in $W^{2,\infty}$ Assumption \ref{assumption_basis}),  
 both rates are smaller than the optimal rates $\frac{\alpha s}{s+1} =2$ and $\frac{2 \alpha s}{s+1} =4$ in Theorem \ref{thm:rate}, where $s=1$ is the degree of B-splines, set by assuming $\phi \in W^{1,\infty}$ by ignoring the discontinuities, and $\alpha=2$ is the order of trapezoidal integration.  
% \FL{We use the spline basis with degree 1 and knot numbers from 2 to 30, and RKHS basis with a dimension ranging from 1 to 40. }\QL{For spline basis, the relative error in RKHS norm is 11.13\%, and is 37.17\% in $L^2(\rhoT)$ norm. For RKHS basis, the relative error in RKHS norm and $L^2(\rhoT)$ are $10.41\%$ and $42.96\%$ respectively.} 

 \begin{figure}[htb]
  \centering 
  \subfigure[The solution $u(x,t)$]{  \includegraphics[width=0.32\textwidth]{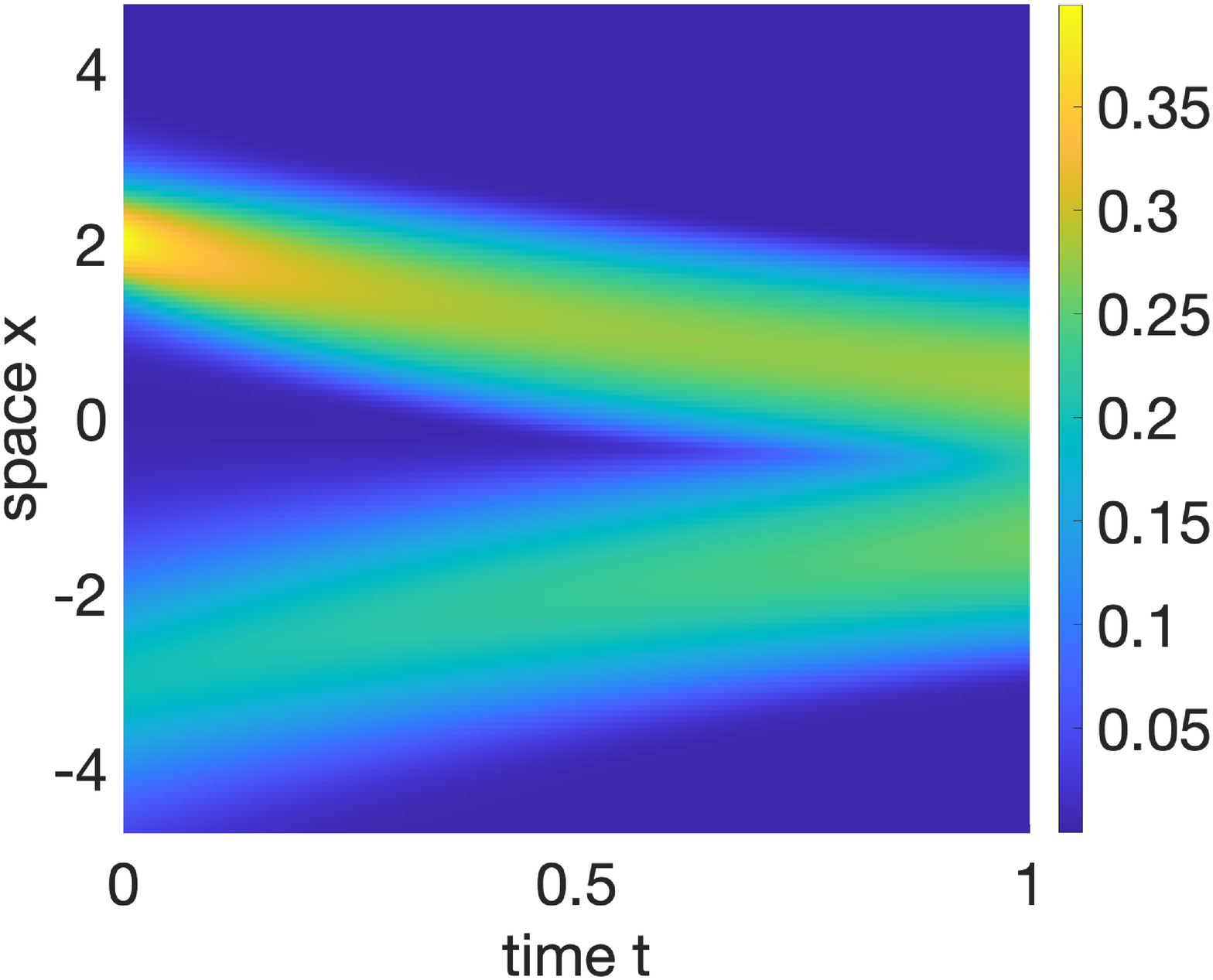}  }
  \subfigure[Wasserstein distance $W(u, \widehat u)( t)$]{    \includegraphics[width=0.3\textwidth]{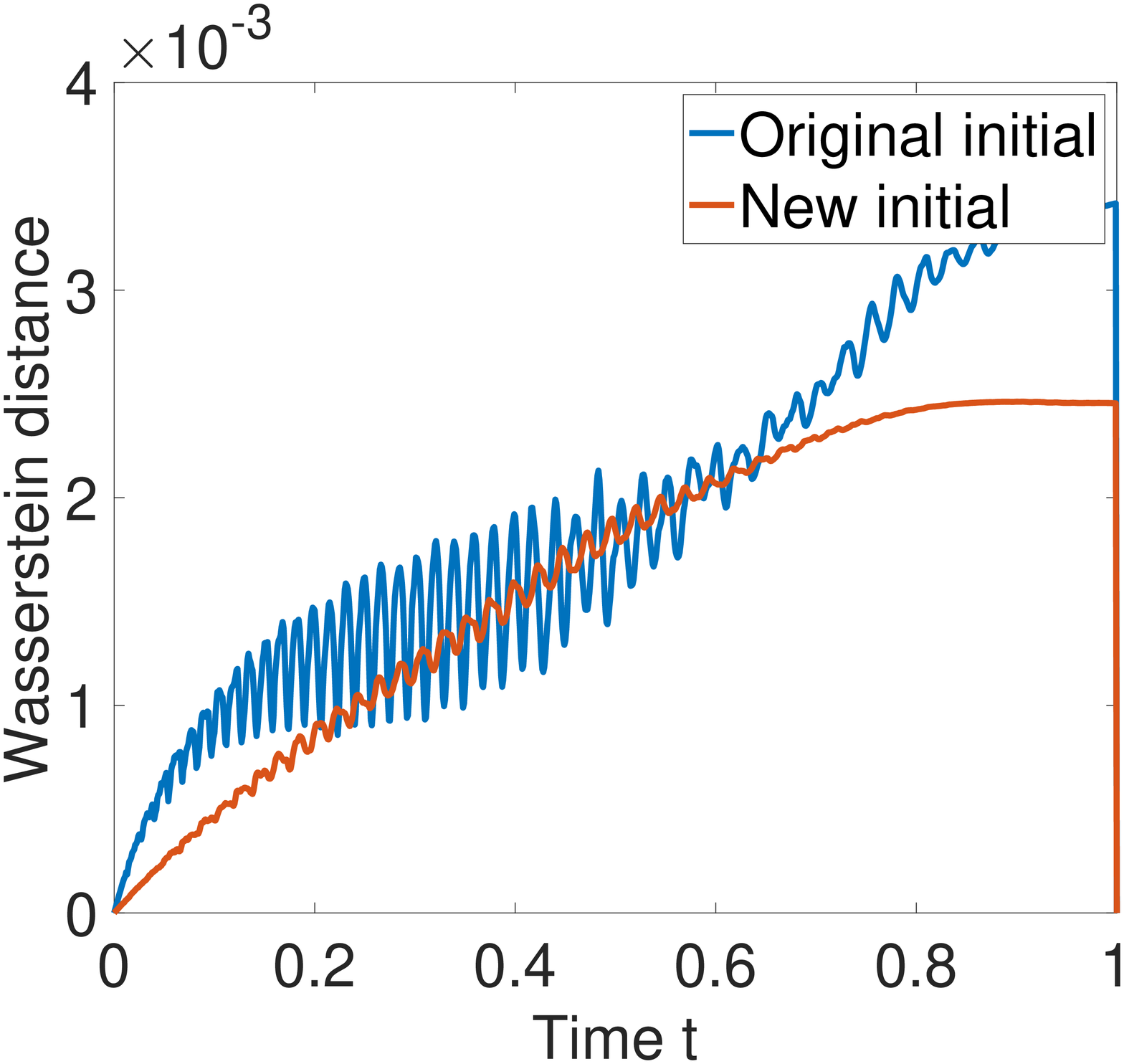}  }\ \ \ 
  \subfigure[Convergence rate of $L^2(\rhoT)$ error]{    \includegraphics[width=0.3\textwidth]{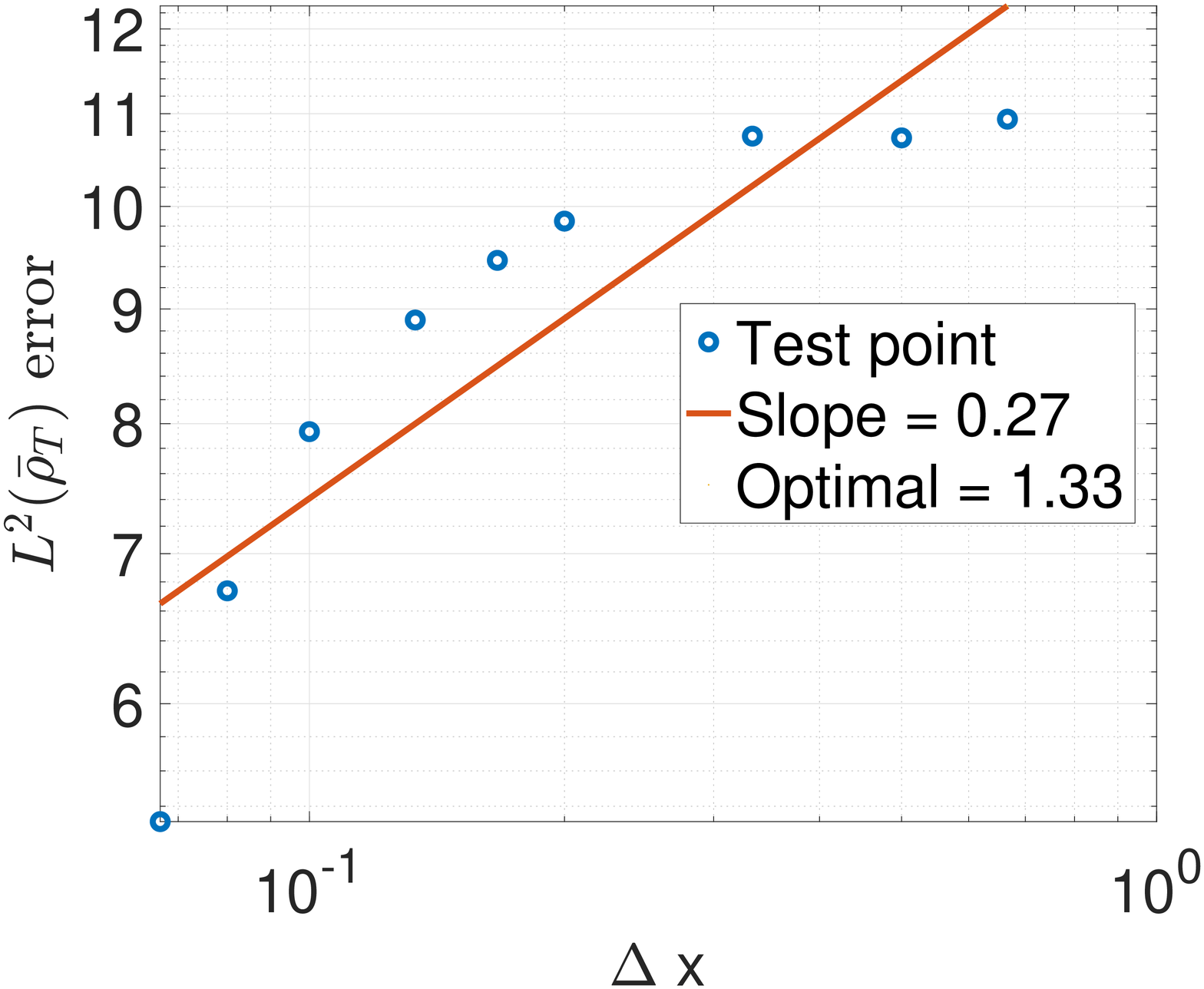}   }
  \subfigure[Estimators of $\phi$]{    \includegraphics[width=0.32\textwidth]{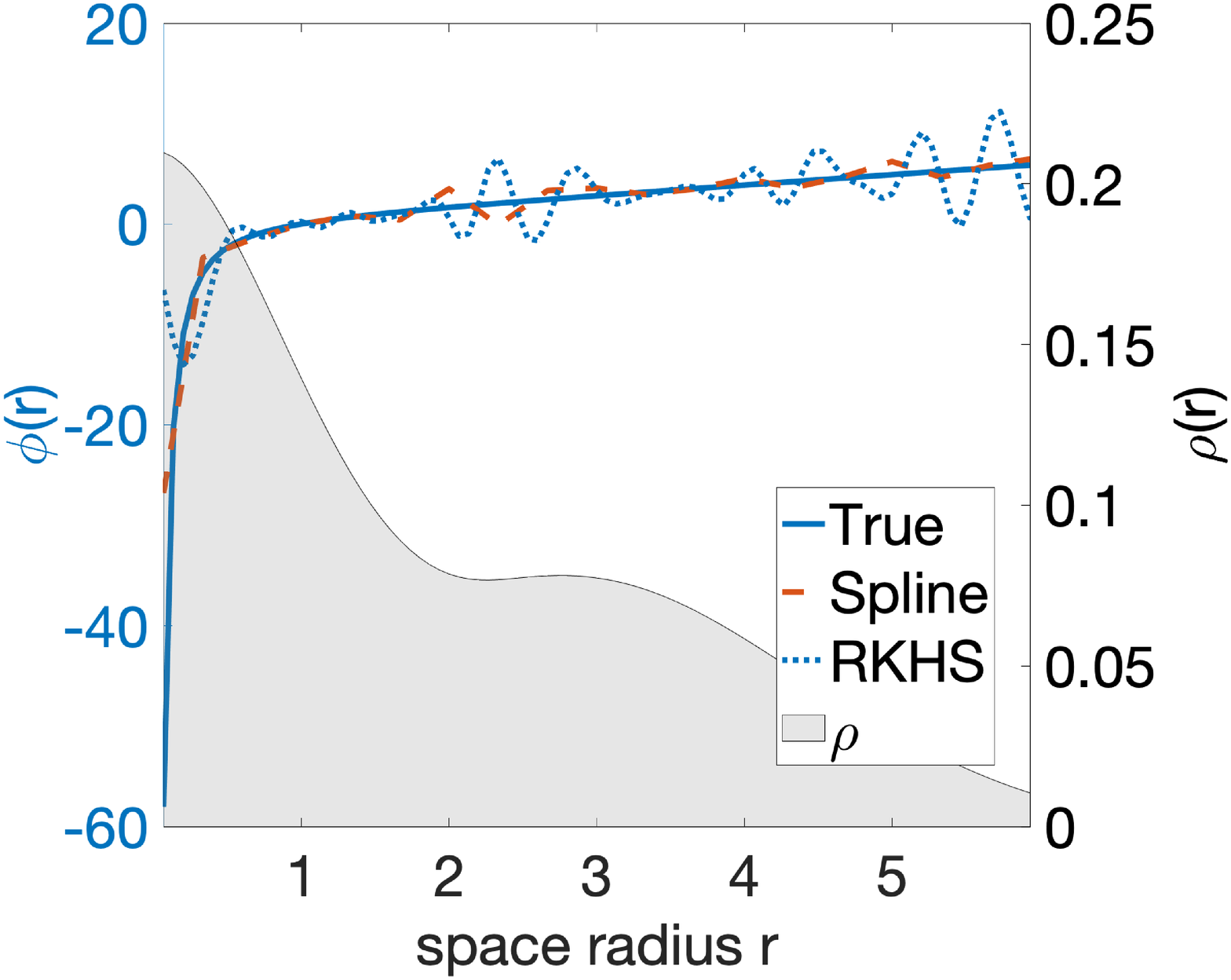}  }
  \subfigure[Free energy]{    \includegraphics[width=0.3\textwidth]{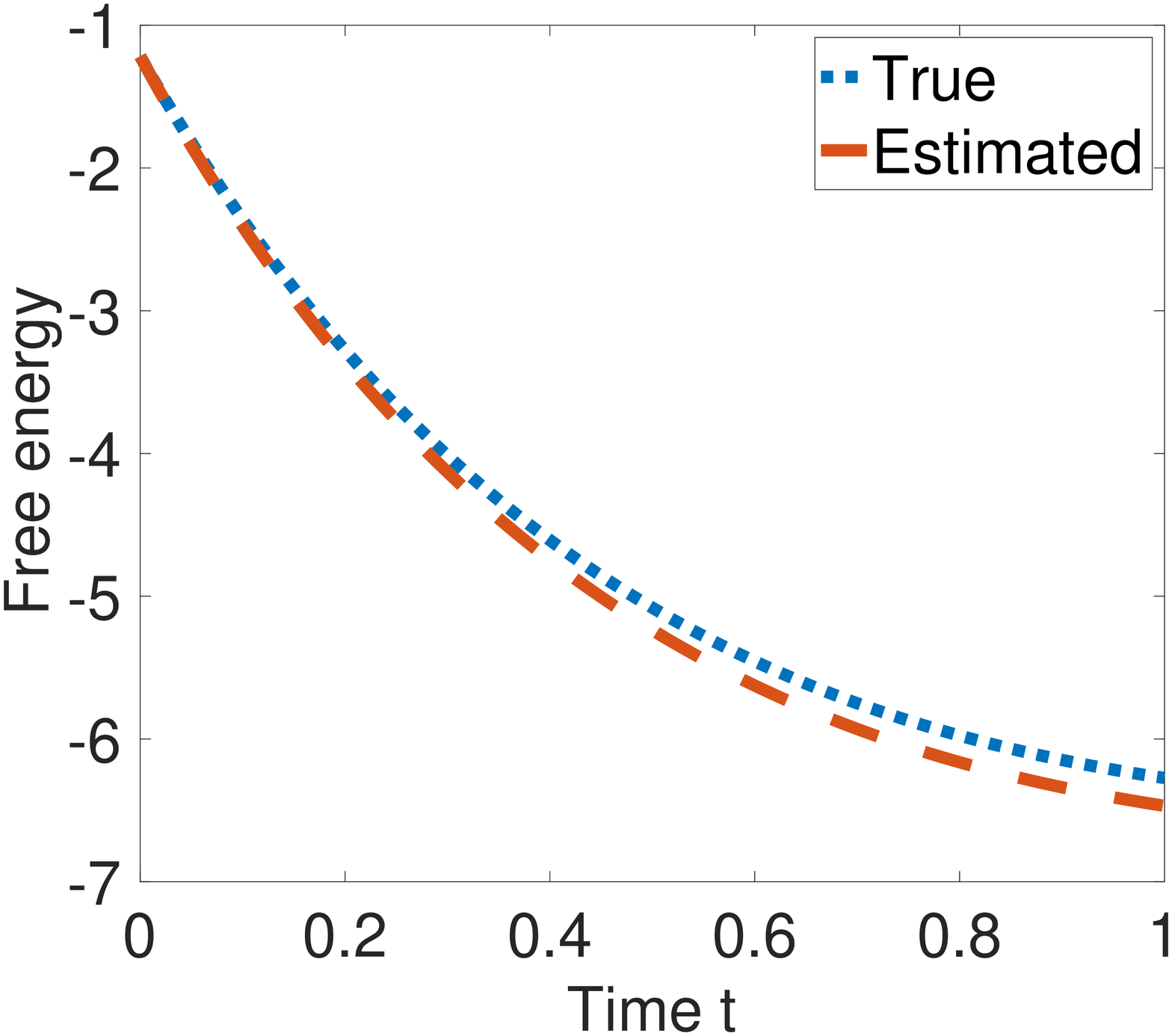}    }\ \ \ 
  \subfigure[Convergence rate of $\mE_{M,L}$ in $\Delta x$]{
    \includegraphics[width=0.3\textwidth]{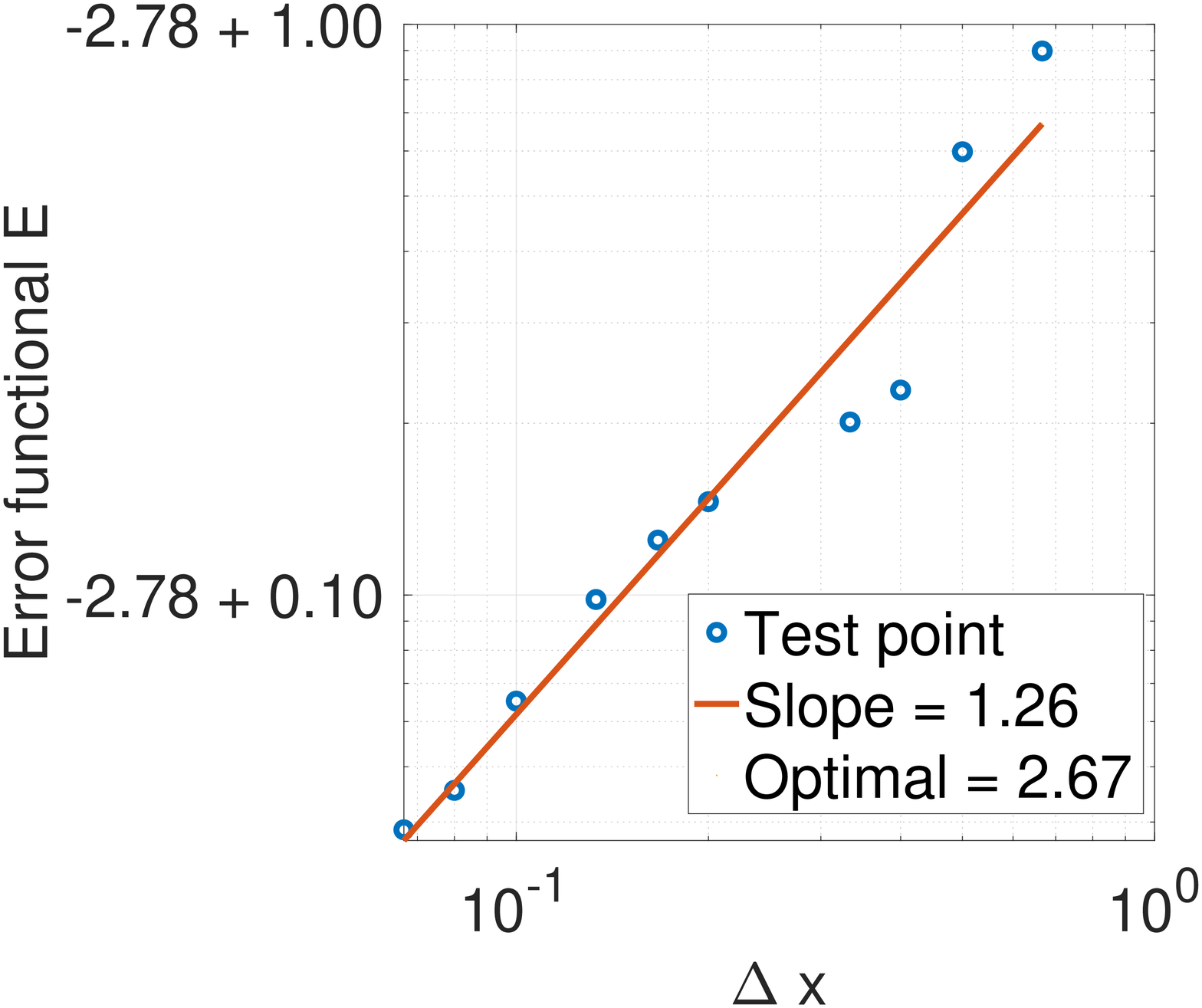} }\vspace{-3mm}
  \caption{Learning result of a repulsion-attraction potential.  In (a) we can see a repulsion-attraction effect. The two cluster tend to get closer, but not merging together because of the repulsion. In (d), we can see that our method tried to learn the singularity at the origin. The relative errors shown in Table \ref{tab:AB}. (b) shows that the Wasserstein distance between the solutions are small. The free energy estimate in (e) is pretty good. This is the intuitive nature of our algorithm. Subfigure (c) and (f) shows the rate of convergence of the estimator $L^2(\rhoT)$ and the error functional. These rates are relatively low due to the singularity of the true kernel (and the optimal rate in Theorem \ref{thm:rate} does not apply). 
  }  \label{fig:AB}
\end{figure}

\subsection{The repulsion-attraction potential} \label{sec:RA}
To model the collision free particles, the repulsion-attraction (RA) potential with singularity at 0 is widely used. We consider the power law RA potential \cite{carrillo2019_AggregationdiffusionEquations}: 
$$ \Phi(x) = \frac{|x|^p}{p} - \frac{|x|^q}{q},\ \ 2\geq p>q>-d.$$
When $0<q<1$, 0 is a singular point of $\phi$. When $q<0$, 0 is also a singular point of $\Phi$. We take $p = 2$ and $q = -0.5$. To show the repulsion clearly, we use a small viscosity constant $v = 0.01$. We take the initial value $u_0$ being the mean of the densities of $\mN(2, 0.25)$ and $\mN(-3, 1)$. We use B-spline basis with degree 1.
%2, matching the power $p=2$. 

Figure \ref{fig:AB} exhibits the estimation results. Sub-figure (a) shows the solution $u(x,t)$, which demonstrate the attraction and repulsion under the influence of diffusion. Subfigure (d) shows that the estimated kernels, either by B-spline or RKHS basis, are close to the true kernel, with relative errors shown in Table \ref{tab:AB}. % For spline basis, the relative error in RKHS norm is 4.38\%, and is 51.83\% in $L^2(\rhoT)$ norm. The optimal knot number is 24. For RKHS basis, the relative error in RKHS norm and $L^2(\rhoT)$ are 2.28\% and 86.96\% respectively, and the optimal basis number is 40. 
\vspace{-3mm}
\begin{table}[H]
 \centering
 {\small
		\caption{ \, Relative errors of estimators in Figure \ref{fig:AB}(d).%, with  $\|\phi\|_{L^2(\rhoT)}=XXX$ and $\|\phi\|_{RKHS}=XXX$.
		 } \label{tab:AB} \vspace{-3mm}
		\begin{tabular}{ c | c | c c  }
		\toprule % \hline
			           &                              &  \multicolumn{2}{c}{Relative error}   \\  
		Basis type  & optimal dimension & in $L^2(\rhoT)$ ($\|\phi\|=10.84$)  & in RKHS  ($\|\phi\|=1.59$)   	 \\ \hline
	         B-spline     & 30                          & 49.06\%                     &  4.36\%        \\
	         RKHS       & 40                       &      86.96\%                 &    2.28\%  \\
			\bottomrule	
		\end{tabular}  
}
\end{table}
The large relative error in $L^2(\rhoT)$ is due to the singularity at the origin and that the measure $\rhoT$ does not reflect the repulsion.  Nevertheless, Subfigure (b) shows that the estimated kernel can reproduce accurate solutions, suggesting that the $L^2(\rhoT)$ norm may not be suitable for the assessment of the estimator of singular kernels.   The slightly oscillating Wasserstein distances indicate that the error in the estimator does not propagate. Subfigure (e) shows that the free energy flow is almost perfectly reproduced. Subfigure (c) and (f) shows relatively low rates of convergence of the estimator in $L^2(\rhoT)$ and the error functional. These rates are relatively low. Due to the singularity of the kernel at the origin, the optimal rate in Theorem \ref{thm:rate} (with $\alpha=2, s=2$) does not apply.

% \subsection{2D learning} We give a final example in the 2D case.  We leave this as future project. Compare it with Monte Carlo. 

%%%%%%%%%%%%%%==== section ======  %%%%%%%%%%
\iffalse
\subsection{Learning from multiple trajectories}
We can combine the information if we are provided multiple trajectories with different initial distribution. Suppose $u^1, \dots, u^m$ are $m$ solutions of the equation \eqref{eq:MFE} with initial distributions $u^1_0, \dots, u^m_0$. We define the combined error functional as 
\begin{align}
\mE'(\psi) & = \sum_{i = 1} ^m \mE^i(\psi) = 
\sum_{i = 1} ^m\rkhs{\psi, \psi}_{G^i}
-2\rkhs{\psi, \phi}_{G^i} \nonumber\\
& = 
\rkhs{\psi, \psi}_{G'}
-2\rkhs{\psi, \phi}_{G'}
\end{align}
where $G^i$ is the corresponding kernel of the solution $u^i$ as in \eqref{eq:GbarT}, and kernel function $G' = \sum_{i=1}^m G_i$. Similar discretization and regularization can be applied after choosing basis. 
\fi

\section{Conclusion and future work}\label{sec:conclusion} 

We have introduced a scalable nonparametric learning algorithm to estimate the interaction kernel from discrete data with performance guarantee. The algorithm learns the kernel on a data-adaptive hypothesis space by least squares with regularization. The estimator is the minimizer of a probabilistic error functional derived from the likelihood of the diffusion process whose Fokker-Planck equation is the mean-field equation. It does not require spatial derivatives of the solution, thus it is suitable for discrete data. We prove that, under regularity conditions, the estimator converges as the space mesh size decreases, in the RKHS space of learning, at a rate optimal in the sense that it is almost the order of the numerical integrator in the evaluation. The empirical error functional converges at twice the rate, and thus this rate can be used for model selection. 

We demonstrate the performance of our algorithm on three typical examples: the opinion dynamics with a piecewise linear kernel, the granular media model with a quadratic kernel and the aggregation-diffusion with a repulsive-attractive kernel. In all the examples, the estimator is accurate, and it can reproduce solutions with small Wasserstein distance to the truth and with almost perfect free energy. For the quadratic kernel, which is smooth, our estimator achieve the optimal rate of convergence. For non-smooth piecewise linear kernel and the singular repulsive-attractive kernel, our estimator converges at sub-optimal rates.  

\bigskip
There are many directions of future research to extend the present work. We mention a few here:
\begin{itemize}
\item Second-order systems of interacting particles/agents and systems with multiple potentials.    
\item Non-radial interaction kernels. Many applications involve non-radial kernels, such as the Biot-Savart kernel \cite{jabin2017_MeanFielda} and the local time kernel for viscous Burgers equation \cite{sznitman1991_TopicsPropagation}. The major issue is the curse of dimensionality in representing high-dimensional functions. We expect the data-adaptive RKHS basis \cite{cucker2007learning} to continue to work. 
\item High-dimensional space.  It become impractical to have data on mesh-grids when the dimension $d$ of the space is large, because the size of mesh-grids increases exponentially in $d$. It is natural to consider data consisting of samples of particles and approximate the error functional by Monte Carlo. Our algorithm applies and the optimal convergence rate would be $s/(2s+1)$. 
\item Partial observations of large systems of interacting agents. When only partial particles of a large system of agents are observed, it is an ill-posed problem to estimate the position of other agents \cite{zhang2020cluster}. By the propagations of chaos, we may view these particles as independent trajectories and estimate the interaction kernel from the SDE of the mean-field equation. 
\end{itemize}
\iffalse
Future directions:
\begin{itemize}
\item high dimensional cases. 
\item system with singular kernels; The vertex model and other. 
\item Measurement error.  The error should not be at pointwise white noise. The observation should be another measure sequence $\{ v(\cdot,t_l) \}$, which adds only a small perturbation to    $\{ u(\cdot,t_l) \}$. 
\item Aggression equation (without diffusion) — what is a good error functional? What is the difference from with diffusion? Can we view it as 0 viscosity limit? 
\item Everything to 2nd-order systems; 
%  + BIG: learn the Boltzman, Landau, NS equations; and learn reduced models for them
\end{itemize}
\fi

\appendix
% \section{Appendix}
\section{Appendix: errors in the numerical integrations}\label{appendixA}
We provide technical bounds on the error of the numerical integrator based on Riemann sum (the Euler scheme). Let's start with a reminder about the error of the Euler scheme. 
\begin{lemma}\label{lemma:Euler_appr}
 Let  $x_m =  m\Delta x$ and $t_l = l\Delta t$ be the mesh given in \eqref{eq:meshData1D}. Suppose that Assumption \rmref{assumption_uxt} holds true. Suppose $f\in W^{1,\infty}(\Omega\times[0,T])$. Then the Euler scheme is of order $\sqrt{\Delta x^2 + \Delta t^2}$, i.e.
  \begin{align}
\overline{\mathcal{D}}(f)& :=    \abs{  \frac{1}{T}  \int_\Omega \int_0^T f(x,t)\ dxdt -    \sum_{l=1}^L \sum_{m=1}^M f(x_m, t_l)\Delta x\Delta t }
    \leq  |\Omega| \norm{\nabla_{x,t} f}_\LOT  (\Delta x + \Delta t), \label{eq:Euler_order} \\
\mathcal{D}_t(f)& :=\abs{ \int_\Omega  f udx  -  \sum_{m=1}^M f(x_m,t) u(x_m,t) \Delta x} \leq  \left( \norm{\nabla f}_\LO + |\Omega| \norm{f}_\LO \norm{\nabla u}_\LO \right) \Delta x. \label{eq:sum4Exp}
\end{align}
\end{lemma}

\begin{proof}  By the midpoint theorem, for $x\in [x_m, x_{m+1}],t\in [t_l, t_{l+1}]$, there exists $(\xi_m, \zeta_l)$ such that 
  \begin{align*}
  |f(x_m, t_l) - f(x,t)|=&  |\nabla_{x,t} f(\xi_m, \zeta_l) \cdot (x-x_m, t - t_l)| \leq  \|\nabla_{x,t} f \|_\LOT     \sqrt{\Delta x^2 + \Delta t^2}.
    \end{align*}  
Note that $\sqrt{\Delta x^2 + \Delta t^2}\leq \Delta x + \Delta t$. Then, \eqref{eq:Euler_order} follows from 
  \begin{align*}
\overline{\mathcal{D}}(f)   \leq&   \frac{1}{T}  \sum_{l=1}^L \sum_{m=1}^M
  \int_{x_m}^{x_{m+1}} \int_{t_{l-1}}^{t_{l}} \abs{f(x_m, t_l) - f(x,t) }dxdt 
  \leq  |\Omega| \norm{\nabla_{x,t} f}_\LOT   \sqrt{\Delta x^2 + \Delta t^2}.
  \end{align*}

Similarly,  \eqref{eq:sum4Exp} follows from
\begin{align*}
\mathcal{D}_t(f)
\leq  &  \sum_{m=1}^M   \int_{x_m}^{x_{m+1}} \left[\abs{f(x,t) - f(x_m) } u(x,t)   + \abs{f(x_m,t)} \abs{u(x,t) - u(x_m,t)} \right] dx \\
\leq  &  \norm{\nabla f}_\LO \Delta x  +  \norm{f}_\LO  \norm{\nabla u}_\LO |\Omega| \Delta x, 
\end{align*}
where the last inequality follows from that $\sum_{m=1}^M   \int_{x_m}^{x_{m+1}}  u(x,t)dx = \int_\Omega u(x,t)dx = 1$ for all $t$. 
\end{proof}

\begin{lemma} Suppose that Assumption \rmref{assumption_basis} holds true. 
The errors of $P_{n, M, L}^i$, $Q_{n, M, L}^i$ and $R_{n, M, L}^i$ in \eqref{eq:PQ}, which approximate $K_{\phi_i} * u$, $\Phi_i* u$ and $\grad\cdot K_{\phi_i}*u$, are bounded by 
  \begin{equation}\label{eq:P_conv_error}
  \begin{aligned}
 \norm{ P_{n, M, L}^i  - K_{\phi_i} * u }_\LOT & \leq    |\Omega|\Delta x \norm{\phi_i}_\WO\norm{u}_\WOT \leq  c^{1,\infty}_\hypspace|\Omega| \norm{u}_\WOT \Delta x, \\
     \norm{R_{n, M, L}^i  - (\grad\cdot K_{\phi_i}) * u }_\LOT & \leq    |\Omega|\Delta x \norm{\phi_i'}_\WO\norm{u}_\WOT \leq  c^{2,\infty}_\hypspace|\Omega| \norm{u}_\WOT \Delta x, \\
   \norm{ Q_{n, M, L}^i  - \Phi_i * u }_\LOT &\leq  |\Omega| (1+\RadiusOmega) c^\infty_\hypspace \norm{u}_\WOT \Delta x. % \leq  c^{\infty}_\hypspace| \Omega| \norm{u}_\WOT)\Delta x. 
  \end{aligned}
  \end{equation} 
\end{lemma}

\begin{proof}
Using the notation $\mathcal{D}_t(f)$ in \eqref{eq:sum4Exp} with $f(\cdot)= K_{\phi_i}(x-\cdot)$, we have,
\begin{align*} 
 & \norm{ P_{n, M, L}^i  - K_{\phi_i} * u }_\LOT =   \sup_{(x,t)\in \Omega\times [0,T]} | \mathcal{D}_t(K_{\phi_i}(x-\cdot))| \\
\leq & \sup_{(x,t)\in \Omega\times [0,T]} \left( \norm{\nabla K_{\phi_i}(x-\cdot)}_\LO + |\Omega| \norm{K_{\phi_i}(x-\cdot)}_\LO \norm{\nabla u}_\LO \right) \Delta x . 
\end{align*}
Recall that we denote $K_{\phi_i}(x) = \phi_i(\abs{x})\frac{x}{|x|}$. For $x\neq 0$, we have $\frac{d}{dx} |x| = \frac{x}{|x|}$ and
\begin{align*}%  \label{Kphi_grad}
  \frac{d}{dx}K_{\phi_i}(x) = &\frac{d}{dx} \brak{\phi_i(\abs{x})\frac{x}{|x|}} % = \phi_i'(|x|)\frac{x}{|x|} \frac{d |x|}{dx} + \phi_i(|x|) \frac{d}{dx}\frac{x}{|x|}\notag\\
  =  \phi_i'(|x|) + \phi_i(|x|)\frac{|x| - x\frac{x}{|x|}}{|x|^2}   = \phi_i'(|x|).
\end{align*} 
Thus, $\norm{\nabla K_{\phi_i}(x-\cdot)}_\LO \leq \norm{\phi_i'}_\LO $ and $\norm{K_{\phi_i}(x-\cdot)}_\LO \leq \norm{\phi_i}_\LO$. 
Together with \eqref{eq:intU=1}, we have 
\begin{align*} 
\norm{ P_{n, M, L}^i  - K_{\phi_i} * u }_\LOT \leq  ( \norm{\phi_i'}_\LO  + |\Omega| \norm{\phi_i}_\LO  \norm{\nabla u}_\LO) \Delta x 
 \leq  |\Omega| \norm{\phi_i}_\WO\norm{u}_\WOT \Delta x.
\end{align*}

Note that $\grad\cdot K_{\phi_i} = \phi_i'(|x|)$. Then, the same argument leads to the estimate for $R_{n, M, L}^i$. 

Similarly, from the definition of  $Q_{n, M, L}^i$ and the notation in \eqref{eq:sum4Exp}, we have
\begin{align*}
&  \norm{ Q_{n, M, L}^i  - \Phi_i * u }_\LOT \leq  \sup_{(x,t)\in \Omega\times [0,T]} | \mathcal{D}_t(\Phi_i(|x-\cdot|))| \\
\leq & \sup_{x\in \Omega} \left[ \norm{\nabla \Phi_i(|x-\cdot|)}_\LO + |\Omega| \norm{\Phi_i(|x-\cdot|)}_\LO \norm{\nabla u}_\LO \right] \Delta x   \\
\leq & \left( 1 + |\Omega|\RadiusOmega  \norm{\nabla u}_\LO \right) \|\phi_i\|_\infty \Delta x 
          \leq  |\Omega| (1+\RadiusOmega) \norm{\phi_i}_\LO\norm{u}_\WOT \Delta x,
\end{align*}
where the second last inequality follows from that $\Phi_i (r)= \int_0^r \phi_i(s)ds$ and $\norm{ \Phi_i}_\infty\leq \|\phi_i\|_\infty \RadiusOmega$. 
\end{proof}

\begin{lemma} \label{lemma:K*phi}
Suppose that Assumption \ref{assumption_basis} holds true. Then, for each $i,j$, 
  \begin{align}
    \norm{K_{\phi_i} * u}_\LOT \leq &  \, \norm{\phi_i}_\LO \leq c^\infty_\hypspace     \label{convnorm}  \\
  \norm{\nabla_{x,t}\brak{K_{\phi_i}*u}}_\LOT \leq & \,  |\Omega| \norm{\nabla_{x,t} u}_\LOT \norm{\phi_i}_\LO \leq |\Omega| \norm{ u}_\WOT c^\infty_\hypspace  ,  \label{helper} \\
  \norm{\nabla_{x,t} \left[  (K_{\phi_i}*u) (K_{\phi_j}*u)  \right]}_\LOT   \leq & \, 2|\Omega| \norm{ u}_\WOT
 (  c^\infty_\hypspace)^2.     \label{eq:Ki*uKj*u} \\
   \norm{\nabla_{x,t}\brak{u \grad\cdot K_{\phi_i}*u}}_\LOT \leq & \,    \norm{ u}_\WOT c^{1,\infty}_\hypspace (1+\norm{u}_\LOT). \label{eq:nablaK*u}  
\end{align}

\end{lemma}
\begin{proof}
Note that $\norm{u(\cdot, t)}_{L^1(\Omega)}=1$ for each $t$. Then equation \eqref{convnorm} follows from that 
  \begin{equation*}
    \norm{K_{\phi_i} * u}_\LOT = \sup_{t\in[0,T]}\norm{K_{\phi_i} * u(\cdot, t)}_\LO \leq  \norm{\phi_i}_\LO \norm{u(\cdot, t)}_{L^1(\Omega)} = \norm{\phi_i}_\LO,
  \end{equation*} 
Equation \eqref{helper} follows from that $\nabla_{x,t}\brak{K_{\phi_i}*u} = K_{\phi_i}*\nabla_{x,t} u$ and $\abs{K_{\phi_i} } \leq \norm{\phi_i}_\LO$. Since
  \begin{align*} 
    &\norm{\nabla_{x,t} \left[  (K_{\phi_i}*u) (K_{\phi_j}*u) \right]}_\LOT 
         \leq   2\left(       \max_{i = 1, \dots, n}  \norm{\nabla_{x,t}\brak{K_{\phi_i}*u}}_\LOT   \right)
                    \left(    \max_{i = 1, \dots, n}     \norm{K_{\phi_i}*u}_\LOT   \right), 
   %  \leq  &2|\Omega|\norm{u}_\LOT\norm{\nabla_{x,t} u}_\LOT   ( c^\infty_\hypspace)^2
  \end{align*}
we obtain \eqref{eq:Ki*uKj*u} from \eqref{convnorm} --\eqref{helper}. 

Note that $\norm{(\grad\cdot K_{\phi_i})*u}_\infty \leq \norm{\phi_i'}_\LO$ and $\norm{\grad_{x,t}\brak{\grad\cdot K_{\phi_i}*u}}_\LOT\leq \norm{\phi_i'}_\LO \norm{\grad_{x,t} u}_\LOT  $, we have
\begin{align*}
\norm{\nabla_{x,t}\brak{u \grad\cdot K_{\phi_i}*u}}_\LOT \leq & \,   \norm{\nabla_{x,t} u}_\LOT \norm{(\grad\cdot K_{\phi_i})*u}_\infty + \norm{u}_\LOT \norm{\grad_{x,t}\brak{\grad\cdot K_{\phi_i}*u}}_\LOT  \\
 \leq & 
\norm{\phi_i'}_\LO \norm{\grad_{x,t} u}_\LOT   (1+\norm{u}_\LOT) \leq  c^{1,\infty}_\hypspace \norm{u}_\WOT  (1+\norm{u}_\LOT). 
\end{align*}
This gives \eqref{eq:nablaK*u}. 
 \end{proof}

 \begin{lemma}\label{lemma:b_time}
 For $I_1^b$ defined in \eqref{eq:I^b}, we have
 \begin{align*}
 I_1^b % &=   \abs{ \frac{1}{T}\int_0^T\int_{\R^d} \partial_tu \Phi_i * u dxdt - \Delta x\Delta t \sum_{l,m=1}^{L,M}  \left[ \widehat{\partial_t u} \Phi_i * u \right] (x_m,t_l)} \\
 & \leq 2 c^\infty_\hypspace\RadiusOmega  \abs{\Omega} (\norm{u}_\WOT^2+ \norm{u}_\WBOT) (\Delta x + \Delta t). 
 \end{align*}
 \end{lemma}
 \begin{proof} 
  Denote $g(x,t) =\Phi_i*u(x,t)$. Note that  $\norm{\Phi_i}_\LO \leq \norm{\phi_i}_\LO \RadiusOmega \leq c^\infty_\hypspace \RadiusOmega$. Then,  
  \[
  \norm{g}_\LOT \leq \norm{\Phi_i}_\LO \leq c^\infty_\hypspace \RadiusOmega; \quad  \norm{\grad_{x,t}g}_\LOT\leq \norm{\Phi_i}_\LO\norm{u}_\WOT \leq   c^\infty_\hypspace \RadiusOmega \norm{u}_\WOT. 
  \]
Note also that $ \widehat{\partial_t u}(x_m,t_l)= \frac{u(x_m,t_{l})-u(x_m,t_{l-1})}{\Delta t}= \partial_t u(x_m,t^*)$ for some $t^*\in [t_l,t_{l+1}]$, we have
$ \abs{\partial_t u (x,t) - \widehat{\partial_t u}(x_m, t_l)  }  \leq  ( \|\partial_{xt}u\|_\infty + \norm{\partial_{tt}u}_\LOT )(\Delta x + \Delta t )$. Thus, 
 \begin{align*}
     %   &  \int_{x_m}^{x_{m+1}} \int_{t_{l-1}}^{t_{l}} g(x,t)\partial_t u (x,t)-  \widehat{\partial_t u}(x_m, t_l) g(x_m, t_l) dxdt  \notag\\
&  \sup_{x\in (x_m, x_{m+1}), t\in (t_{l-1}, t_{l}) } \abs{g(x,t)\partial_t u (x,t)-  \widehat{\partial_t u}(x_m, t_l) g(x_m, t_l) } \notag \\
\leq &  \sup_{x\in (x_m, x_{m+1}), t\in (t_{l-1}, t_{l}) } \left[  \abs{ g(x,t) - g(x_m, t_l) } \norm{\partial_t u }_\LOT 
                                                       + \norm{g}_\LOT \abs{\partial_t u (x,t) - \widehat{\partial_t u}(x_m, t_l)} \right] \notag \\
\leq  &  (\norm{\grad_{x,t}g}_\LOT \norm{\partial_t u}_\LOT + \norm{g}_\LOT  \|u\|_{2,\infty})  (\Delta x + \Delta t).
 \end{align*}
The, note that $\norm{\grad_{x,t}g}_\LOT \norm{\partial_t u}_\LOT + \norm{g}_\LOT  \|u\|_{2,\infty} \leq 2 c^\infty_\hypspace\RadiusOmega (\norm{u}_\WOT^2+ \norm{u}_\WBOT)$, we have 
 \begin{align*}
 I_1^b & \leq \frac{1}{T} \sum_{l=1}^L \sum_{m=1}^M
                   \abs{  \int_{x_m}^{x_{m+1}} \int_{t_{l-1}}^{t_{l}} g(x,t)\partial_t u (x,t)-  \widehat{\partial_t u}(x_m, t_l) g(x_m, t_l) dxdt } \\
   & \leq \frac{1}{T} \sum_{l=1}^L \sum_{m=1}^M\Delta x  \Delta t
             \sup_{x\in (x_m, x_{m+1}), t\in (t_{l-1}, t_{l}) } \abs{g(x,t)\partial_t u (x,t)-  \widehat{\partial_t u}(x_m, t_l) g(x_m, t_l) }  \\
 & \leq 2 c^\infty_\hypspace\RadiusOmega \abs{\Omega}(\norm{u}_\WOT^2+ \norm{u}_\WBOT)  (\Delta x + \Delta t).
 \end{align*}  
 \end{proof}
 %We remark that the second order derivatives of the solution are necessary to control of the Riemann sum approximation of the integrals. With stronger regularity on the solution and higher-order approximations of the integrals than the Euler scheme, one can obtain higher order convergence in space and time. In the other direction, since these integrals are expectations and they can be approximated by samples, we expect to removed these regularity assumptions in forthcoming research. 

\textbf{Acknowledgments.} {QL and FL are grateful for supports from NSF-1821211 and NSF-1913243. FL would like to thank Mauro Maggioni, P-E Jabin and Zhenfu Wang for helpful discussions. }

\bibliographystyle{siamplain}

% \bibliography{ref_learnMFE,ref_Meanfield,ref_FeiLU,learning_dynamics,LearningTheory,ref_stocParticleSys,OpinionDS_ref,ref_Bspline}
\bibliography{all_ref}

\end{document}